%% file: main.tex
\documentclass[10pt,twocolumn,letterpaper]{article}

\usepackage{cvpr}      % To produce the REVIEW version
\usepackage[accsupp]{axessibility}  % Improves PDF readability for those with disabilities.
\input{preamble}
\definecolor{cvprblue}{rgb}{0.21,0.49,0.74}
\usepackage[pagebackref,breaklinks,colorlinks,allcolors=cvprblue]{hyperref}

\usepackage{url} % Provides better formatting of URLs.
\usepackage[utf8]{inputenc} % Allows Turkish characters.
\usepackage{booktabs} % Allows the use of \toprule, \midrule and \bottomrule in tables for horizontal lines
\usepackage{graphicx}
\usepackage{xcolor}
\usepackage{multirow}
\usepackage{adjustbox}
\usepackage{colortbl}   % 单元格着色
\usepackage{pifont}   % 提供 \ding 符号
\usepackage[T1]{fontenc}
\usepackage{bbm}
\usepackage{listings}
\newcommand{\cmark}{\textcolor{black}{\ding{51}}} 
\newcommand{\xmark}{\textcolor{black}{\ding{55}}} 
\definecolor{lightpurple}{RGB}{230, 220, 250}
\usepackage{booktabs}
\usepackage{adjustbox}
\usepackage[table]{xcolor}
\usepackage{makecell}
\usepackage{array}
\usepackage{caption}
\definecolor{HeaderGray}{RGB}{245,245,245}
\definecolor{BandGray}{RGB}{244,246,250}
\definecolor{MyRed}{RGB}{208,16,32}
\definecolor{MyBlue}{RGB}{0,92,184}
\definecolor{BandOrange}{RGB}{255,229,204}
\definecolor{BandPurple}{RGB}{225,220,255}
\definecolor{BandGreen}{RGB}{212,241,215}
\definecolor{BandPink}{RGB}{244,212,229}

\newcommand{\best}[1]{\textbf{\textcolor{MyRed}{#1}}}
\newcommand{\second}[1]{\textbf{\textcolor{MyBlue}{#1}}}

\newcommand{\tighttabsetup}{%
  \setlength{\tabcolsep}{2.8pt}%
  \renewcommand{\arraystretch}{1.06}%
  \aboverulesep=0.35ex \belowrulesep=0.35ex
  \abovetopsep=0.25ex \belowbottomsep=0.25ex
}
\newcommand{\tablesize}{\fontsize{7.1}{8.3}\selectfont}
\newcommand{\FourDatasetHeader}[4]{%
  \rowcolor{HeaderGray}
  \multirow{2}{*}{Method} &
  \multicolumn{5}{c}{#1} &
  \multicolumn{5}{c}{#2} &
  \multicolumn{5}{c}{#3} &
  \multicolumn{5}{c}{#4}\\
  \cmidrule(lr){2-6}\cmidrule(lr){7-11}\cmidrule(lr){12-16}\cmidrule(lr){17-21}
  \rowcolor{HeaderGray}
  & PSNR$\uparrow$ & SSIM$\uparrow$ & LPIPS$\downarrow$ & FID$\downarrow$ & DISTS$\downarrow$
  & PSNR$\uparrow$ & SSIM$\uparrow$ & LPIPS$\downarrow$ & FID$\downarrow$ & DISTS$\downarrow$
  & PSNR$\uparrow$ & SSIM$\uparrow$ & LPIPS$\downarrow$ & FID$\downarrow$ & DISTS$\downarrow$
  & PSNR$\uparrow$ & SSIM$\uparrow$ & LPIPS$\downarrow$ & FID$\downarrow$ & DISTS$\downarrow$\\
}
\newcommand{\MetricsFiveHead}{PSNR$\uparrow$ & SSIM$\uparrow$ & LPIPS$\downarrow$ & FID$\downarrow$ & DISTS$\downarrow$}

\newcommand{\FourDatasetHeaderIQA}[4]{%
  \rowcolor{HeaderGray}
  \multirow{2}{*}{Method} &
  \multicolumn{5}{c}{#1} &
  \multicolumn{5}{c}{#2} &
  \multicolumn{5}{c}{#3} &
  \multicolumn{5}{c}{#4}\\
  \cmidrule(lr){2-6}\cmidrule(lr){7-11}\cmidrule(lr){12-16}\cmidrule(lr){17-21}
  \rowcolor{HeaderGray}
  & NIQE$\downarrow$ & MUSIQ$\uparrow$ & MANIQA$\uparrow$ & CLIPIQA$\uparrow$ & TOPIQ$\uparrow$
  & NIQE$\downarrow$ & MUSIQ$\uparrow$ & MANIQA$\uparrow$ & CLIPIQA$\uparrow$ & TOPIQ$\uparrow$
  & NIQE$\downarrow$ & MUSIQ$\uparrow$ & MANIQA$\uparrow$ & CLIPIQA$\uparrow$ & TOPIQ$\uparrow$
  & NIQE$\downarrow$ & MUSIQ$\uparrow$ & MANIQA$\uparrow$ & CLIPIQA$\uparrow$ & TOPIQ$\uparrow$\\
}

\newcommand{\SectionBand}[2]{%
  \rowcolor{#1}\multicolumn{21}{l}{\textit{#2}}\\[-0.5ex]
}

\newcommand{\TriSeriesHeader}[3]{%
\multirowcell{2}{\textbf{Method}} &
\multicolumn{5}{c}{#1} &
\multicolumn{5}{c}{#2} &
\multicolumn{5}{c}{#3} \\
\cmidrule(lr){2-6}\cmidrule(lr){7-11}\cmidrule(lr){12-16}
\multicolumn{1}{c}{} & \MetricsFiveHead & \MetricsFiveHead & \MetricsFiveHead \\
}

\newcommand{\SinSeriesHeader}[1]{%
\multirowcell{2}{\textbf{Method}} &
\multicolumn{5}{c}{#1} \\
\cmidrule(lr){2-6}
\multicolumn{1}{c}{} & \MetricsFiveHead \\
}

\usepackage{booktabs,adjustbox,array,siunitx,xcolor}
\usepackage{caption} % 可选：更紧凑的标题间距
\newcolumntype{L}[1]{>{\raggedright\arraybackslash}m{#1}}

\definecolor{myblue}{RGB}{66,133,244}
\definecolor{mygreen}{RGB}{51,168,83}
\definecolor{myyellow}{RGB}{251,188,3}
\definecolor{myred}{RGB}{234,67,53}
\definecolor{mygrey}{RGB}{95,99,104}
\definecolor{mypup}{RGB}{153,0,204}

\definecolor{RubineRed}{HTML}{D10056}
\definecolor{RoyalBlue}{HTML}{4169E1}
\definecolor{ForestGreen}{HTML}{228B22}
\definecolor{BurntOrange}{HTML}{CC5500}
\definecolor{Purple}{HTML}{800080}
\definecolor{Teal}{HTML}{008080}
\definecolor{Brown}{HTML}{8B4513}
\usepackage[most]{tcolorbox}
\tcbset{
  promptstyle/.style={
    breakable, enhanced,
    sharp corners,
    colback=black!2, colframe=black,
    boxrule=0.4pt,
    left=6pt, right=6pt, top=4pt, bottom=4pt,
    before skip=6pt, after skip=6pt,  % 紧凑页间距，避免大片留白
  }
}
\newenvironment{promptbox}{\begin{tcolorbox}[promptstyle]}{\end{tcolorbox}}

\newcommand{\prompttt}{\scriptsize\ttfamily}

\usepackage{amsmath,amssymb,mathtools,amsthm}
\usepackage{bm}
\newtheorem{theorem}{Theorem}
\newtheorem{definition}{Definition}
\newtheorem{lemma}{Lemma}
\newtheorem{proposition}{Proposition}
\newtheorem{corollary}{Corollary}
\newtheorem{assumption}{Assumption}
\theoremstyle{remark}\newtheorem{remark}{Remark}

\title{FAPE-IR: Frequency-Aware Planning and Execution Framework \\ for All-in-One Image Restoration}

\author{
Jingren Liu\textsuperscript{1}\thanks{Equal contribution.},
Shuning Xu\textsuperscript{2}\footnotemark[1],
Qirui Yang\textsuperscript{1}\footnotemark[1],
Yun Wang\textsuperscript{3},
Xiangyu Chen\textsuperscript{4}\footnotemark[2],
Zhong Ji\textsuperscript{1}\thanks{Corresponding author.} \\
\textsuperscript{1}Tianjin University, \textsuperscript{2}University of Macau, \textsuperscript{3}City University of Hong Kong \\
\textsuperscript{4}Institute of Artificial Intelligence (TeleAI), China Telecom
}

\begin{document}
\maketitle
\input{sec/0_abstract}    
\input{sec/1_introduction}
\input{sec/2_related_works}

\input{sec/3_plan_then_restore}
\input{sec/4_experiments}
\input{sec/5_conclusion}

{
    \small
    \bibliographystyle{ieeenat_fullname}
    \bibliography{main}
}

\input{sec/X_suppl}

\end{document}

%% file: sec/0_abstract.tex
\begin{abstract}
All-in-One Image Restoration (AIO-IR) aims to develop a unified model that can handle multiple degradations under complex conditions. However, existing methods often rely on task-specific designs or latent routing strategies, making it hard to adapt to real-world scenarios with various degradations. We propose FAPE-IR, a Frequency-Aware Planning and Execution framework for image restoration. It uses a frozen Multimodal Large Language Model (MLLM) as a planner to analyze degraded images and generate concise, frequency-aware restoration plans. These plans guide a LoRA-based Mixture-of-Experts (LoRA-MoE) module within a diffusion-based executor, which dynamically selects high- or low-frequency experts, complemented by frequency features of the input image. To further improve restoration quality and reduce artifacts, we introduce adversarial training and a frequency regularization loss. By coupling semantic planning with frequency-based restoration, FAPE-IR offers a unified and interpretable solution for all-in-one image restoration. Extensive experiments show that FAPE-IR achieves state-of-the-art performance across seven restoration tasks and exhibits strong zero-shot generalization under mixed degradations. Code is available at \url{https://github.com/Programmergg/FAPE-IR}.
\end{abstract}

%% file: sec/1_introduction.tex
\vspace{-4ex}
\section{Introduction}
\label{sec:intro}

\begin{figure}[t]
\centering
\includegraphics[width=0.93\linewidth]{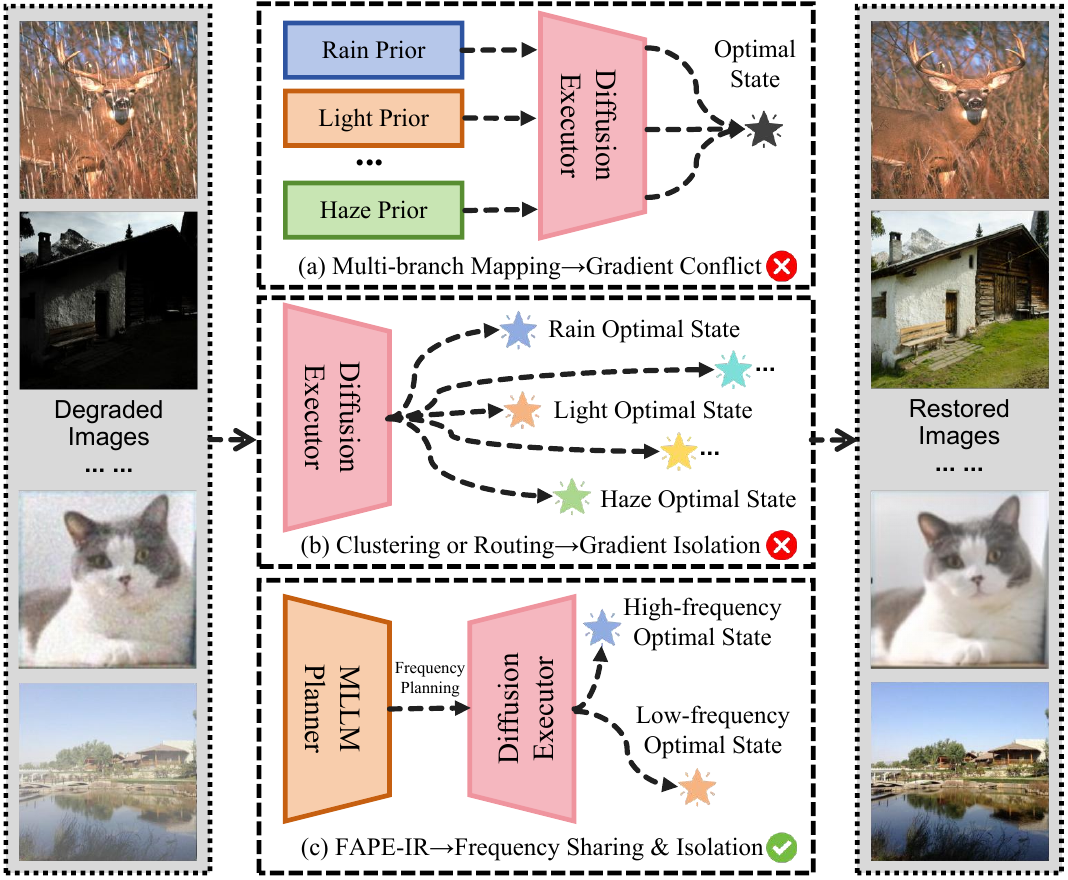}
\caption{Comparison of AIO-IR methods: (a) multi-branch mappings with task-level priors; (b) task-specific routing/clustering; (c) our FAPE-IR, unifying understanding and restoration. }
\label{figs:fig1}
\vspace{-4ex}
\end{figure}

Image restoration seeks to recover clean images from degradations such as rain, snow, fog, blur, low-light, and noise, and is fundamental to safety- and quality-critical applications in medical imaging \cite{chen2025all,wang2025anatomy}, consumer photography \cite{xu2024image}, and autonomous driving \cite{wang2025moerl,jeong2025robust,lin2025jarvisir}. 
Traditional pipelines typically employ specialized architectures or optimization strategies for each type of degradation \cite{ren2023multiscale,yi2023diff,zhang2017beyond,kawar2022denoising}.
% Traditional pipelines typically engineer custom architectures or optimization schemes for each degradation \cite{ren2023multiscale,yi2023diff,zhang2017beyond,kawar2022denoising}. 
Although recent single-degradation models report strong results on curated benchmarks \cite{xie2025diffusion,wang2025learning,zhou2025prodehaze,lan2025efficient,guo2025compression}, real-world images rarely present a single known corruption; degradations are often unknown, co-occurring, and subject to distribution shift, which undermines stability and controllability. This gap motivates All-in-One Image Restoration (AIO-IR), which aims to handle multiple degradations within a single model \cite{chen2025unirestore,luo2025visual,tian2025degradation,shi2024resfusion,ye2024learning,chen2025exploring,pu2025lumina,chen2024learning}. Yet, most current efforts still assume task separability \cite{luo2025visual,ai2024multimodal} or tether restoration to reconstruction pipelines driven by task-level priors \cite{wang2025enhancement,wu2025debiased}, preventing models from autonomously perceiving and disentangling degradation attributes. As a result, model generalization to real-world scenarios remains inadequate.

Within current AIO-IR methods, a prevailing strategy is to employ multi-branch mappings (see Figure~\ref{figs:fig1}(a)), injecting task-level conditions via explicit priors (e.g., prompts, embeddings, or specialized encoders) and learning the restoration paths accordingly \cite{potlapalli2023promptir,li2023prompt,ma2023prores,luo2023controlling,chen2025unirestore,zhou2025unires}. This design is prone to severe cross-task interference. In a unified model with shared parameters, gradients from different degradations conflict, making it difficult to reach task-specific optima simultaneously. Moreover, both training and inference depend on prior labels or textual inputs, incurring substantial annotation and prompt-engineering costs. A complementary strategy clusters features or employs specialized routing mechanisms within the model (see Figure~\ref{figs:fig1}(b)) \cite{cui2024adair,tian2025degradation,zamfir2025complexity,zheng2024selective,wang2025m2restore,cheng2025unildiff}. Because tasks are strongly separated in latent space, their learning processes become nearly independent, which hampers the discovery of shared structures and compositionality across degradations. Consequently, robustness to compound degradations is limited, and generalization suffers. More critically, both lines lack semantic understanding and rely on opaque degradation pipelines, hindering content-aware adaptation and interpretability. 
These limitations call for a unified, semantics-aware AIO-IR framework in which degradation understanding and restoration are coupled, and task knowledge is both disentangled and shared in an adaptive way.

To instantiate such a framework, we adopt diffusion-based image restoration, which has recently become a popular choice for AIO-IR due to its strong generative priors \cite{potlapalli2023promptir,chen2025unirestore,zhou2025unires,cui2024adair,zheng2024selective,wang2025m2restore,cheng2025unildiff}. 
This formulation naturally models diverse degradations, yet prior work typically treats diffusion as a task-conditioned generator driven by labels or hand-crafted priors, leaving degradation modeling opaque and under-exploiting cross-degradation interactions.
We instead regard diffusion as the execution engine in a unified multimodal understanding–generation paradigm. In this paradigm, degradations are first parsed and organized, and then restored.
Inspired by recent unified models~\cite{xie2024show,lin2025uniworld,wu2025janus,deng2025emerging,chen2025blip3,yang2025mmada}, we instantiate FAPE-IR, a multimodal large language model (MLLM) + diffusion framework for low-level vision (see Figure~\ref{figs:fig1}(c)). 
An MLLM explicitly parses image content and degradation semantics and outputs understanding and planning tokens that condition a diffusion executor to perform content-adaptive restoration in the frequency domain, grouping tasks into high- and low-frequency regimes, thereby encouraging sharing within same-frequency tasks and isolating conflicting ones.

Concretely, FAPE-IR couples a multimodal planner with a diffusion-based executor to mitigate cross-task gradient conflict and isolation from a frequency perspective. In the planning stage, we employ a label-free procedure to extract as many degradation-related low-level features as possible, and introduce several well-designed instructions that constrain the model to output: the degradation types present in the image, the primary frequency bands requiring restoration, the underlying causes, and the proposed restoration pipeline. This design yields effective, per-image understanding that more accurately guides the restoration process. In the execution stage, tokens produced by the planner, augmented with high-/low-frequency visual features from SigLIP-v2 \cite{tschannen2025siglip} and a VAE \cite{flux2024}, guide the restoration of the entire image. Furthermore, we propose a Frequency-Aware LoRA-MoE architecture, where text tokens from the planner and the high-/low-frequency signals extracted from intermediate executor's features drive the MoE gating to select among high-/low-frequency experts dynamically and interpretably. For the overall training recipe, we provide theoretical and empirical evidence that adversarial training on diffusion-pretrained weights can yield higher fidelity and fewer artifacts than other fine-tuning methods under our setting.
Building on this, we further introduce an energy-based frequency regularization loss so that each frequency expert is optimized in its most suitable band.

Our main contributions are summarized as follows:
(i) We introduce FAPE-IR, a unified understanding–generation AIO-IR framework that couples a multimodal planner with a diffusion-based executor, enabling semantics-driven, content-adaptive restoration across diverse degradations. (ii) We develop a frequency-aware modeling strategy that dynamically routes information between high- and low-frequency experts, mitigating cross-task interference while promoting shared representations among related tasks. (iii) We conduct comprehensive experiments on a wide range of single-degradation and mixed-degradation benchmarks, demonstrating that FAPE-IR achieves state-of-the-art or competitive performance and exhibits strong zero-shot generalization to unseen compound degradations.

%% file: sec/2_related_works.tex
\section{Related Works}
\label{sec:related_works}

\subsection{Diffusion-based AIO-IR}
In recent years, diffusion-based AIO-IR methods have advanced rapidly. Existing approaches fall into two families: (i) multi-branch mappings and (ii) clustering or routing.
The former relies on injecting task-level priors into a shared backbone via task encoders or textual/latent prompts. PromptIR~\cite{potlapalli2306promptir}, Prompt-in-Prompt~\cite{li2023prompt}, and InstructIR~\cite{conde2024instructir} unify degradations through prompts or instructions. ProRes~\cite{ma2023prores} and DA-CLIP~\cite{luo2023controlling} supply content/degradation embeddings via visual prompts or CLIP~\cite{radford2021learning} controllers and fuse them with cross-attention. UniRestore~\cite{chen2025unirestore} injects task cues into diffusion, whereas UniRes~\cite{zhou2025unires} composes task specialists at sampling time.
The latter adapts internal representations by clustering features or routing to experts. AdaIR~\cite{cui2024adair} and DFPIR~\cite{tian2025degradation} mine frequency cues and reshape separable, degradation-aware features. AMIRNet~\cite{zhang2023all} performs label-free, layer-wise clustering with domain alignment. Within diffusion, DiffUIR~\cite{zheng2024selective} disentangles shared/private factors for selective routing. MoCE-IR~\cite{zamfir2025complexity} and M2Restore~\cite{wang2025m2restore} use complexity-aware MoE to activate only necessary experts.
Overall, these strategies improve efficiency but lack sample-level planning, relying on fixed task-level designs and ignoring whether tasks should share or isolate knowledge. This leads to under-shared representations, misrouting, and artifacts under unknown or composite degradations, limiting open-world robustness. To address this, we propose FAPE-IR, a unified understanding–generation framework.

\subsection{Unified Models}
A central goal of multimodal AI is to develop general-purpose foundation models that can jointly understand and generate diverse modalities within a single framework. Recent unified models largely follow three paradigms. Multimodal autoregression methods, such as Chameleon~\cite{team2024chameleon}, Emu3~\cite{wang2024emu3}, and UGen~\cite{tang2025ugen}, tokenize images into a shared discrete space and train a single decoder-only Transformer, but this tokenization degrades fine details and incurs long generation latency, limiting their suitability for precise low-level restoration. Multimodal diffusion approaches instead unify modalities at the denoising stage: MMaDA~\cite{yang2025mmada} employs a modality-agnostic diffusion backbone with unified post-training; LaViDa~\cite{li2025lavida} uses elastic masked diffusion for understanding, editing, and generation; Dimple~\cite{yu2025dimple} couples an autoregression warm-up with discrete diffusion and parallel confident decoding; and UniDisc~\cite{swerdlow2025unified} casts joint text--image modeling as discrete diffusion in a shared token space. However, this paradigm remains nascent, with unresolved issues in consistency and inference efficiency. A third line, MLLM+Diffusion frameworks, decouple high-level multimodal understanding from generation: UniWorld-V1~\cite{lin2025uniworld} conditions diffusion on high-resolution semantic features from MLLMs and contrastive encoders; BAGEL~\cite{deng2025emerging} integrates decoder-only MLLMs with diffusion heads in an ``integrated Transformer'' for both understanding and image generation; Janus~\cite{wu2025janus} decouples visual encoding for understanding versus generation under a unified backbone; and BLIP3-o~\cite{chen2025blip3} uses CLIP-based semantic spaces with diffusion transformers and staged pretraining to balance image understanding and generation. Yet these conditioning interfaces are primarily designed for high-level semantic editing or creation, rather than pixel-accurate, artifact-free restoration. Our FAPE-IR framework follows the MLLM+Diffusion paradigm but is specialized for low-level image restoration, where fine-grained controllability and artifact suppression are central.

%% file: sec/3_plan_then_restore.tex
\begin{figure*}[t]
\centering
\includegraphics[width=0.93\linewidth]{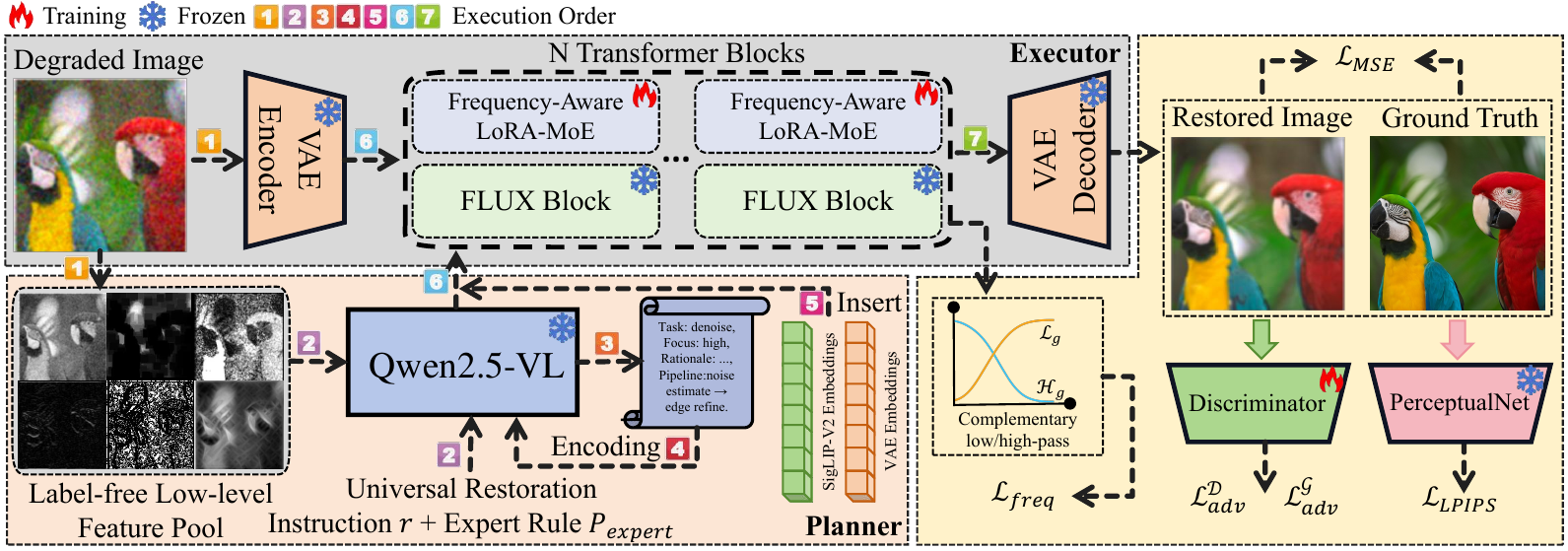}
\caption{Overview of the proposed FAPE-IR framework.}
\label{figs:fig2}
\vspace{-3ex}
\end{figure*}

\section{Methodology}
\label{sec:plan_then_restore}
As shown in Figure~\ref{figs:fig2}, FAPE-IR adopts a planning and execution paradigm that decomposes image restoration into a frequency-aware planner and a band-specialized restoration executor. In the remainder of this section, we first describe the frequency-aware planner, including the construction of a label-free low-level feature pool, instructions and planning, and the encoding of frequency-aligned understanding tokens (Section~\ref{sec:freqaware}). We then present the proposed Frequency-Aware LoRA-MoE architecture in executor (Section~\ref{sec:restoration-detailed}). Finally, we introduce the adversarial training strategy and the overall loss formulation, which couples frequency-aware routing with perceptual and adversarial supervision (Section~\ref{sec:adv_training}).

\subsection{Frequency-aware Planner}
\label{sec:freqaware}
\textbf{Label-free Low-level Feature Pool.} Given a degraded input image $c$, we first construct a label-free low-level feature pool that characterizes its degradation pattern from a frequency perspective (details in Appendix). Specifically, we compute a vector of simple image statistics $P_{\text{hints}}$ directly from pixel values, without relying on degradation labels or auxiliary metadata. Each component of $P_{\text{hints}}$ corresponds to one of seven representative degradation types: for \textbf{rain}, the strength of oriented streaks; for \textbf{snow}, the density of small bright blobs; for \textbf{noise}, luminance and chroma variations in flat regions; for \textbf{blur}, the responses of Laplacian and gradient operators; for \textbf{haze}, dark-channel and saturation statistics; for \textbf{low light}, global luminance; and for \textbf{super-resolution}, the spatial dimensions $H \times W$ of the input. These lightweight global statistics can be computed in a single pass and require no supervision. The resulting vector $P_{\text{hints}}$ is then provided to the planner as frequency-aware visual information describing the input degradation.

\vspace{-3ex}
\paragraph{Instructions and Planning.}
We consider a low-level restoration task space $\mathcal{T}$ that covers seven representative degradations, including low-light enhancement, dehazing, desnowing, deraining, deblurring, denoising, and super-resolution. In FAPE-IR, Qwen2.5-VL~\cite{bai2025qwen2} serves as a multimodal planner that turns frequency-aware visual cues and textual instructions into a structured, frequency-explicit restoration plan. Concretely, the planner is conditioned on three types of inputs: a universal restoration instruction $r$ that describes the general goal of recovering a clean and sharp image (details in Appendix); an expert rule $P_{\text{expert}}$ that encodes the task taxonomy and prior knowledge about which degradations affect which frequency bands (also detailed in Appendix); and the concise visual hints $P_{\text{hints}}$ derived from the label-free low-level feature pool. Based on these inputs, Qwen2.5-VL produces a parse-friendly textual output, which we then parse into a tuple
$FP = (\hat{t}, \hat{f}, \mathcal{R}, \mathcal{E})$, where $\hat{t}$ is the selected task from $\mathcal{T}$, $\hat{f}$ indicates whether the restoration should primarily focus on high- or low-frequency content, $\mathcal{R}$ summarizes the intended restoration pipeline in natural language, and $\mathcal{E}$ explains the reasoning behind this choice. This human-readable plan acts as a routing signal for downstream high- and low-frequency experts, making the decision process transparent and interpretable even under composite degradations.

\vspace{-3ex}
\paragraph{Encoding.}
To ensure the planner’s decisions effectively guide the diffusion-based executor, we use Qwen2.5-VL’s encoding capability to compile the planner decisions $FP$ into compact, frequency-aligned understanding tokens $h$.
To mitigate the gap between understanding tokens and the executor’s generation tokens, we enrich the conditioning tokens with high- and low-frequency visual information. Concretely, we use Qwen2.5-VL’s visual placeholders to inject tokens produced by SigLIP-v2 $\mathcal{E}_{\text{sig}}(c)$ and the VAE $\mathcal{E}_{\text{vae}}(c)$, replacing the executor's original T5 encoding for conditioning $h_{\text{cond}} = \text{insert}\left(h, \mathcal{E}_{\text{sig}}(c), \mathcal{E}_{\text{vae}}(c)\right)$.
In addition, to retain precise, human-interpretable semantics for the subsequent Frequency-Aware LoRA-MoE module, we similarly use text placeholders to extract the textual tokens $h_{\text{text}} = h[:, \text{slot}_{\text{text}}, :]$.

\subsection{Restoration Executor}
\label{sec:restoration-detailed}
Given a degraded image $c$, we employ a restoration executor based on the FLUX transformer~\cite{flux2024} to refine the VAE latent $z$ into a restored latent $\hat{z}$, guided jointly by the textual tokens $h_{\text{text}}$ and the conditioning tokens $h_{\text{cond}}$. The VAE decoder then maps $\hat{z}$ to the restored image $\hat{x}$. 

To meet constraints on model size and optimization time, FAPE-IR adopts a parameter-efficient LoRA-MoE executor with two experts specialized for high- and low-frequency bands. Building on this design, we introduce a Frequency-Aware LoRA-MoE module with a dual-end gating mechanism: a planner-side gate uses the frequency-aligned understanding tokens $h_{\text{text}}$ to propose an expert selection, while an executor-side gate inspects intermediate high- and low-frequency generation tokens to refine that choice. As illustrated in Figure~\ref{figs:fig3}, this design conditions the MoE router jointly on semantic and spectral cues, leading to more precise expert assignment and improved robustness to frequency-band drift during optimization.

\subsubsection{Frequency-Aware LoRA-MoE}

\begin{figure}[t]
\centering
\includegraphics[scale=0.5]{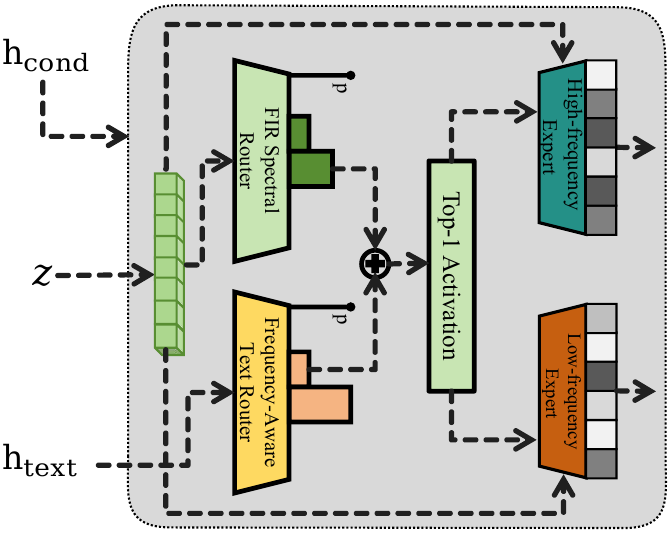}
\caption{Frequency-Aware LoRA-MoE architecture.}
\label{figs:fig3}
\vspace{-4ex}
\end{figure}

% \noindent\textbf{Frequency-Aware Text Router.}
% Since the planner yields $h_{\text{text}}\in\mathbb{R}^{B\times K\times D}$  while the LoRA-MoE gate operates on $\mathbb{R}^{B\times L\times D}$, we right-pad $h_{\text{text}}$ along the token axis to length $L$ ($K \leq L$), leaving $D$ unchanged, so that the gating network can be applied token-wise over a sequence of consistent length. Then, we implement the text gating function as a token-wise fully connected layer on the padded text tokens, followed by a softmax over experts to obtain the text gating weights $w_{text} = \mathrm{Softmax}\big(W_t\cdot\text{Padding}(h_{\text{text}})\big)$.

\noindent\textbf{Frequency-Aware Text Router.}
Since the planner yields $h_{\text{text}}\in\mathbb{R}^{B\times K\times D}$, where $B$ denotes the batch size, $K$ denotes the planner token length, and $D$ denotes the embedding dimension, while the LoRA-MoE gate operates on representations in $\mathbb{R}^{B\times L\times D}$, where $L$ denotes the LoRA-MoE gate token length, we right-pad $h_{\text{text}}$ along the token axis to length $L$ ($K \leq L$), leaving $D$ unchanged, so that the gating network can be applied token-wise over a sequence of consistent length. Then, we implement the text gating function as a token-wise fully connected layer on the padded text tokens, followed by a softmax over experts to obtain the text gating weights $w_{\text{text}} = \mathrm{Softmax}\big(W_t\cdot\text{Padding}(h_{\text{text}})\big)$.

\noindent\textbf{FIR Spectral Router.} 
To compensate for the limitations of high-level semantic tokens in gating, we add a visual-frequency path that separates the executor’s generation tokens into low/high-frequency components using a depthwise FIR low-pass filter \cite{makandar2015image,chen2019drop}. Let $h_{\text{gen}}\in\mathbb{R}^{B\times L\times D}$ denote the per-block input. For an odd kernel size $K$ and a symmetric, normalized 1D Gaussian $g\in\mathbb{R}^{K}$ with $\sum_k g_k=1$, we define a depthwise 1D convolution along the token axis as $h_{\text{low}}=\mathcal{L}_g(h_{\text{gen}})$ and $h_{\text{high}}=h_{\text{gen}}-h_{\text{low}}$.

Next, to allocate the proportion of energy between the low- and high-frequency generation tokens, we compute their relative energy and derive the high–low gating weights. Specifically, for each token we set $e_{\text{low}}=\lVert h_{\text{low}}\rVert_2^2$ and $e_{\text{high}}=\lVert h_{\text{high}}\rVert_2^2$, then $p_{\text{low}}= \frac{e_{\text{low}}}{e_{\text{low}}+e_{\text{high}}}$ and $p_{\text{high}}=1-p_{\text{low}}$; a temperature-scaled softmax over $[p_{\text{low}},p_{\text{high}}]$ yields the high–low gating weights $w_{\text{visual}}$.

\noindent\textbf{Softmax-gated Merged Weights.}
Subsequently, we fuse the textual and spectral gating weights with non-negative, learnable scalars $\lambda_s$: $\tilde{\alpha} \;=\; \lambda_s\,w_{\text{text}} \;+\; (1-\lambda_s)\,w_{\text{visual}}$. Finally, we apply top-1 routing to obtain the final frequency-band selection, i.e., $\alpha=\mathrm{Top1}(\tilde{\alpha})$, where $\alpha$ is a one-hot vector indicating the chosen band.

With per-expert coefficients $\alpha_i$ produced by the fused and sparsified router, the FLUX projection is updated as:
\vspace{-1ex}
\begin{equation}
\label{eq:lora_moe_param}
W' \;=\; W \;+\; \sum_{i=1}^{N} \alpha_i \, A_i B_i,
\end{equation}
where $(A_i,B_i)$ is a rank-$r_i$ LoRA adapter for the $i$-th frequency expert and the backbone $W$ is frozen.

\subsection{Adversarial Training}
\label{sec:adv_training}
To improve efficiency at both training and inference, our FAPE-IR replaces the flow-matching fine-tuning objective~\cite{lipman2022flow} commonly used in unified models with adversarial training. In Theorem~\ref{thm:comp-fidelity} (Appendix), we show that, compared with auto-regression~\cite{team2024chameleon,sun2024generative,wang2024emu3} and flow-matching~\cite{lin2025uniworld,deng2025emerging,wu2025janus,chen2025blip3} losses, adversarial training attains fewer artifacts. Next, we detail FAPE-IR’s training paradigm in Figure~\ref{figs:fig2}.

\noindent \textbf{Discriminator.}
Given a ground truth image $x$ (restored image $\hat{x}$), the frozen SigLIP-v2 discriminator $\mathcal{F}_{\text{sig}}$ yields (i) a sequence of hidden states from which we take $L$ token maps at increasing depths and reshape them into spatial feature maps $\{\mathbf{f}^{(l)}\in\mathbb{R}^{C_l\times H_l\times W_l}\}_{l=1}^{L}$, and (ii) a pooled representation $\mathbf{p}\in\mathbb{R}^{D}$ from the final layer. To more comprehensively judge the input, we attach a multi-level discriminator head that mines spatial evidence across scales.

\noindent \textbf{Multi–level Discriminator Head.}
The discriminator head $\mathcal{H}_\psi$ processes each spatial map with a shallow, spectrally normalized convolutional path that includes anti-aliasing BlurPool downsampling, and processes the pooled token:
\begin{equation}
\mathbf{s}^{(l)} \;=\; \mathcal{H}^{(l)}_\psi\!\left(\mathbf{f}^{(l)}\right)\in\mathbb{R}^{H'_l\times W'_l},
s^{\text{pool}} \;=\; \mathcal{H}^{\text{pool}}_\psi(\mathbf{p}).
\end{equation}
Each spatial score map $\mathbf{s}^{(l)}$ is spatially averaged to a scalar $\bar{s}^{(l)}=\frac{1}{H'_l W'_l}\sum_{i,j}\mathbf{s}^{(l)}_{ij}$; we then aggregate levelwise with uniform weights:
$D(x)\;=\;\frac{1}{L+1}\Big(\sum_{l=1}^{L}\bar{s}^{(l)} \;+\; s^{\text{pool}}\Big).$
This design encourages consistency between local structures (via multi-scale spatial paths) and global semantics (via the pooled path), while remaining lightweight.

\noindent \textbf{Discriminator Loss.}
Following the adversarial training, we optimize the critic with the standard discriminator loss:
\begin{equation}
\mathcal{L}_{\text{adv}}^{\mathcal{D}} = -\mathbb{E}_{x}[\log D(x)] - \mathbb{E}_{\hat{x}}[\log (1 - D(\hat{x}))].
\end{equation}

\noindent \textbf{Generator Loss.}
To minimize over-generation and distortion, we do not concatenate additional noise channels into the inputs of the FLUX transformer. Finally, we train $\mathcal{G}_\theta$ under a composite adversarial training loss.
\begin{equation}
\small
\mathcal{L}_{\text{adv}} =
\underbrace{\alpha\|\hat{x}-x\|_2^2}_{\mathcal{L}_\text{MSE}}
\underbrace{+\,\beta\|\phi(\hat{x})-\phi(x)\|_2^2}_{\mathcal{L}_\text{LPIPS}}
\underbrace{-\,\lambda\,\mathbb{E}\!\left[D(\hat{x})\right]}_{\mathcal{L}_\text{adv}^{\mathcal{G}}}.
\end{equation}

While $\mathcal{L}_{\text{adv}}$ promotes pixel fidelity, perceptual similarity, and adversarial realism, it remains agnostic to how spectral content is allocated across experts. To align training with the frequency-aware routing in Subsection~\ref{sec:freqaware}–\ref{sec:restoration-detailed}, we add a single frequency regularizer on adapter outputs that penalizes out-of-band synthesis. This preserves objective parsimony, only one additional scalar term, while explicitly enforcing low/high-band specialization.

\noindent\textbf{Frequency Regularizer.}
To make the two LoRA–MoE experts specialize in complementary bands, we penalize out-of-band energy on their adapter outputs. Let $\mathcal{L}_g$ be a depthwise FIR low-pass along the token axis, and define its complementary high-pass $\mathcal{H}_g \triangleq I-\mathcal{L}_g$. For the low/high-frequency experts’ outputs $y_{\text{low}}$ and $y_{\text{high}}$, we minimize $\mathcal{L}_{\mathrm{freq}}=\mathrm{mean}\!\big[\|\mathcal{H}_g(y_{\text{low}})\|_2^2+\|\mathcal{L}_g(y_{\text{high}})\|_2^2\big]$. In practice, $g$ is a fixed Gaussian kernel stored as a buffer (non-trainable), and the loss is summed across layers.

At this point, the overall FAPE-IR objective is:
\begin{equation}
\mathcal{L}_{\text{Total}}
=\mathcal{L}_{\text{adv}}
+\gamma\,\mathcal{L}_{\mathrm{freq}}.
\end{equation}

In general, by marrying a frequency-aware planner with band-specialized LoRA–MoE and an  adversarial training objective augmented by a single frequency regularizer, our FAPE-IR reconciles interpretability with fidelity.

%% file: sec/4_experiments.tex
\section{Experiments}
\subsection{Experimental Setup}

\begin{table*}[t]
\centering
\tablesize
\tighttabsetup
\caption{Unified comparison across six AIO-IR task series. Each series shows five metrics (PSNR↑, SSIM↑, LPIPS↓, FID↓, DISTS↓). The best results are highlighted in \textbf{\best{red}}, and the second-best results are shown in \textbf{\second{blue}}.}
\label{tab:IR_series}
\begin{adjustbox}{width=\textwidth}
\begin{tabular}{@{} c *{15}{c} @{}}
\toprule

% -------- 上面三个：Rain / Noise / Blur(Dynamic) --------
\TriSeriesHeader{Deraining}{Denoising}{Deblurring}
\toprule
PromptIR~\cite{potlapalli2306promptir} & \second{21.94} & \second{0.73} & \second{0.30} & 100.17 & \second{0.21} & 31.27 & \second{0.84} & \second{0.16} & 57.82 & \second{0.14} & -- & -- & -- & -- & -- \\
FoundIR~\cite{li2024foundir} & 19.68 & 0.68 & 0.37 & 118.29 & 0.24 & 30.85 & 0.80 & 0.21 & 67.08 & 0.22 & 28.35 & 0.85 & 0.21 & 42.03 & 0.16 \\
DFPIR~\cite{tian2025degradation} & 19.27 & 0.68 & 0.37 & 104.13 & 0.27 & 31.31 & 0.83 & 0.17 & 61.80 & 0.15 & \second{30.82} & \best{0.90} & \second{0.15} & \second{28.23} & \second{0.12} \\
MoCE-IR~\cite{zamfir2025complexity} & 21.42 & 0.71 & 0.32 & 101.21 & 0.23 & \best{31.63} & \second{0.84} & \second{0.16} & \second{56.74} & \second{0.14} & 21.67 & 0.74 & 0.28 & 48.85 & 0.19 \\
AdaIR~\cite{cui2024adair} & 21.64 & 0.71 & 0.32 & \second{100.07} & 0.23 & \second{31.45} & \second{0.84} & 0.17 & 58.05 & \second{0.14} & 22.23 & 0.75 & 0.28 & 49.64 & 0.19 \\
\rowcolor{blue!8}
FAPE-IR (Ours) & \best{28.30} & \best{0.84} & \best{0.09} & \best{21.55} & \best{0.07} & 30.34 & \best{0.87} & \best{0.10} & \best{30.33} & \best{0.09} & \best{30.91} & \second{0.88} & \best{0.10} & \best{16.53} & \best{0.07} \\

\midrule

\TriSeriesHeader{Desnowing}{Dehazing}{Low-light enhancement}
\toprule
PromptIR~\cite{potlapalli2306promptir} & -- & -- & -- & -- & -- & \second{21.94} & \second{0.90} & \second{0.14} & \second{24.26} & \second{0.11} & -- & -- & -- & -- & -- \\
FoundIR~\cite{li2024foundir} & 23.79 & \second{0.78} & 0.28 & 29.31 & 0.17 & 13.65 & 0.75 & 0.31 & 34.41 & 0.21 & 14.44 & 0.67 & 0.32 & 86.46 & 0.22 \\
DFPIR~\cite{tian2025degradation} & 20.78 & 0.74 & 0.30 & 34.86 & 0.18 & 13.54 & 0.73 & 0.30 & 29.43 & 0.21 & \best{25.92} & \best{0.90} & \second{0.16} & 52.15 & 0.13 \\
MoCE-IR~\cite{zamfir2025complexity} & 23.82 & 0.76 & 0.28 & 30.47 & \second{0.16} & 15.54 & 0.80 & 0.22 & 27.98 & 0.15 & 23.36 & \second{0.89} & \second{0.16} & \second{44.61} & \second{0.12} \\
AdaIR~\cite{cui2024adair} & \second{24.19} & 0.77 & \second{0.27} & \second{26.82} & \second{0.16} & 19.89 & 0.88 & 0.16 & 25.71 & 0.12 & 24.51 & 0.86 & 0.17 & 50.78 & 0.13\\
\rowcolor{blue!8}
FAPE-IR (Ours) & \best{30.29} & \best{0.88} & \best{0.08} & \best{1.49} & \best{0.06} & \best{33.85} & \best{0.97} & \best{0.04} & \best{7.21} & \best{0.04} & \second{25.34} & \best{0.90} & \best{0.11} & \best{40.30} & \best{0.09}\\
\bottomrule
\end{tabular}
\end{adjustbox}
\vspace{-4.5ex}
\end{table*}

\begin{table}[t]
\centering
\tablesize
\tighttabsetup
\caption{Unified comparison across SR task series. Each series shows five metrics. The best results are marked in \textbf{\best{red}}, and the second-best results are shown in \textbf{\second{blue}}.}
\label{tab:SR_series}
\begin{adjustbox}{scale=1.15}
\begin{tabular}{@{} c *{5}{c} @{}}
\toprule
% -------- 上面三个：Rain / Noise / Blur(Dynamic) --------
\SinSeriesHeader{SR}
\toprule
StableSR~\cite{wang2024exploiting} & 25.10 & 0.73 & 0.29 & \second{110.18} & 0.22 \\
DiffBIR~\cite{lin2024diffbir} & 25.65 & 0.65 & 0.39 & 124.53 & 0.25 \\
SeeSR~\cite{wu2024seesr} & 26.60 & 0.74 & 0.31 & 122.10 & 0.23 \\
PASD~\cite{yang2024pixel} & \second{26.87} & 0.75 & 0.29 & 120.46 & \second{0.21} \\
OSEDiff~\cite{wu2024one} & 26.49 & \second{0.76} & \second{0.28} & 117.55 & \second{0.21} \\
PURE~\cite{wei2025perceive} & 24.47 & 0.64 & 0.38 & 121.07 & 0.24 \\
\rowcolor{blue!8}
FAPE-IR (Ours) & \best{28.53} & \best{0.85} & \best{0.19} & \best{85.82} & \best{0.15} \\
\bottomrule
\end{tabular}
\end{adjustbox}
\vspace{-5ex}
\end{table}

\noindent\textbf{Training Data.}
We construct a training corpus covering seven restoration tasks, including deblurring, dehazing, low-light enhancement, deraining, desnowing, denoising, and super-resolution.
Weather-related degradations (deraining, desnowing, dehazing) are sourced from Snow100K-Train~\cite{liu2018desnownet}, Rain100-L/H-Train~\cite{yang2019joint}, and the NTIRE 2025 Challenge dataset~\cite{li2025ntire}.
For dehazing, we use a curated $10\mathrm{K}$ subset of OTS/ITS-Train~\cite{li2018benchmarking} combined with URHI-Train~\cite{fang2024real}.
Deblurring utilizes RealBlur-Train (R/J)~\cite{rim2020real}, GoPro-Train, and its gamma-corrected variant~\cite{nah2017deep}.
Low-light enhancement is trained on LOL-v2-Train~\cite{yang2021sparse}.
For denoising, BSD400~\cite{martin2001database} and WaterlooED~\cite{ma2016waterloo} are corrupted with Gaussian noise at $\sigma\in\{15,25,50\}$.
For super-resolution, we follow prior work~\cite{wu2024one,wu2024seesr,wei2025perceive} and adopt LSDIR~\cite{li2023lsdir} plus a $10\mathrm{K}$-image FFHQ subset~\cite{karras2019style}, with degradations synthesized via Real-ESRGAN~\cite{wang2021real}.

\noindent\textbf{Testing Data.}
We evaluate our method on a comprehensive collection of benchmarks that span both synthetic and real-world degradations across deblurring, dehazing, low-light enhancement, deraining, desnowing, denoising, and super-resolution.
For deblurring, we adopt the official test splits of RealBlur-J/R~\cite{rim2020real}, GoPro, and GoPro\_gamma~\cite{nah2017deep}. 
Dehazing evaluation is conducted using the ITS validation set and the URHI test split~\cite{li2018benchmarking}. 
Low-light enhancement is assessed on the test sets of LOL-v1~\cite{wei2018deep} and LOL-v2~\cite{yang2021sparse}. 
Deraining is benchmarked on Rain100-L/H~\cite{yang2019joint}, OutDoor~\cite{li2019heavy}, and RainDrop~\cite{qian2018attentive}. 
For desnowing, we use the Snow100K-L and Snow100K-S test subsets~\cite{liu2018desnownet} for large-scale qualitative comparison. 
Denoising performance is measured on BSD68~\cite{martin2001database} and Urban100~\cite{huang2015single}, each corrupted with additive Gaussian noise at $\sigma\in\{15,25,50\}$. 
Super-resolution is evaluated on RealSR~\cite{ji2020real} and DRealSR~\cite{wei2020component} at $\times2$ and $\times4$ scales. 
We follow standard evaluation protocols: RealSR is tested using the center-crop strategy~\cite{wei2025perceive}, while images from other benchmarks are resized to $512\times512$.

\noindent\textbf{Optimization.} We train our FAPE-IR framework using the Prodigy optimizer~\cite{mishchenko2023prodigy} for the main model and AdamW~\cite{loshchilov2017decoupled} for the discriminator head. The base learning rate for AdamW is initialized at $1e-4$ and cosine-annealed during training. All models are trained for $200\mathrm{K}$ steps with a batch size of 1 on $512\times512$ inputs across all datasets. Training is conducted on $8\times$ NVIDIA H200 GPUs. In the one-step setting, we use $\alpha=50.0$, $\beta=5.0$, $\lambda=0.5$, $\gamma=1\times10^{-3}$, and a diffusion timestep $t=300$. 

\subsection{Quantitative Results}
We conduct a quantitative comparison between FAPE-IR and recent AIO-IR methods, including PromptIR~\cite{potlapalli2306promptir} (NeurIPS’23), FoundIR~\cite{li2024foundir} (ICCV’25), AdaIR~\cite{cui2024adair} (ICLR’25), DFPIR~\cite{tian2025degradation} (CVPR’25), and MoCE-IR~\cite{zamfir2025complexity} (CVPR’25), as well as some SR methods: StableSR~\cite{wang2024exploiting} (IJCV’24), DiffBIR~\cite{lin2024diffbir} (ECCV’24), SeeSR~\cite{wu2024seesr} (CVPR’24), PASD~\cite{yang2024pixel} (ECCV’24), OSEDiff~\cite{wu2024one} (NeurIPS’24), and PURE~\cite{wei2025perceive} (ICCV’25). \textit{We evaluate all tasks on a single trained model rather than training for each task and testing them individually.}

\begin{table}[t]
  \centering
  \begingroup
  \setlength{\tabcolsep}{4pt}
  \renewcommand{\arraystretch}{1.2}
  \caption{Complexity comparison. All methods are evaluated on $512\times512$ inputs, and inference is measured on an H200 GPU. Param (G) denotes the inference memory footprint.}
  \label{tab:flops-time-transposed}
  \begin{adjustbox}{max width=\linewidth}
    % m{1.8cm} 可按需调宽
    \begin{tabular}{>{\centering\arraybackslash}m{1.8cm} *{6}{c}}
      \toprule
      \multirow[c]{2}{*}{\textbf{Metric}} &
      \multicolumn{6}{c}{\textbf{Methods}} \\
      \cmidrule(lr){2-7}
      & FoundIR~\cite{li2024foundir}
      & DFPIR~\cite{tian2025degradation}
      & MoCE-IR~\cite{zamfir2025complexity}
      & AdaIR~\cite{cui2024adair}
      & PURE~\cite{wei2025perceive}
      & Ours \\
      \midrule
      Time (s)  & 0.18 & 0.08 & 0.10 & 0.10 & 201.67 & 1.57 \\
      Param (G) & 3.36 & 3.92 & 2.11 & 3.49 & 26.22 & 38.92 \\
      \bottomrule
    \end{tabular}
  \end{adjustbox}
  \endgroup
  \vspace{-4ex}
\end{table}

\begin{figure*}[t]
\centering
\includegraphics[width=\linewidth]{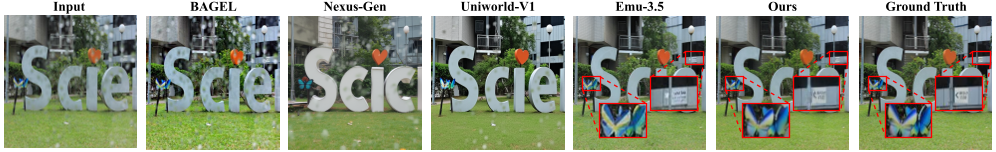}
\caption{Comparison among unified models, including BAGEL~\cite{deng2025emerging}, Nexus-Gen~\cite{zhang2504nexus}, Uniworld-V1~\cite{lin2025uniworld}, and Emu3.5~\cite{cui2025emu3}.}
\label{fig:qual4}
\vspace{-2ex}
\end{figure*}

\begin{figure}[t]
\centering
\includegraphics[width=\linewidth]{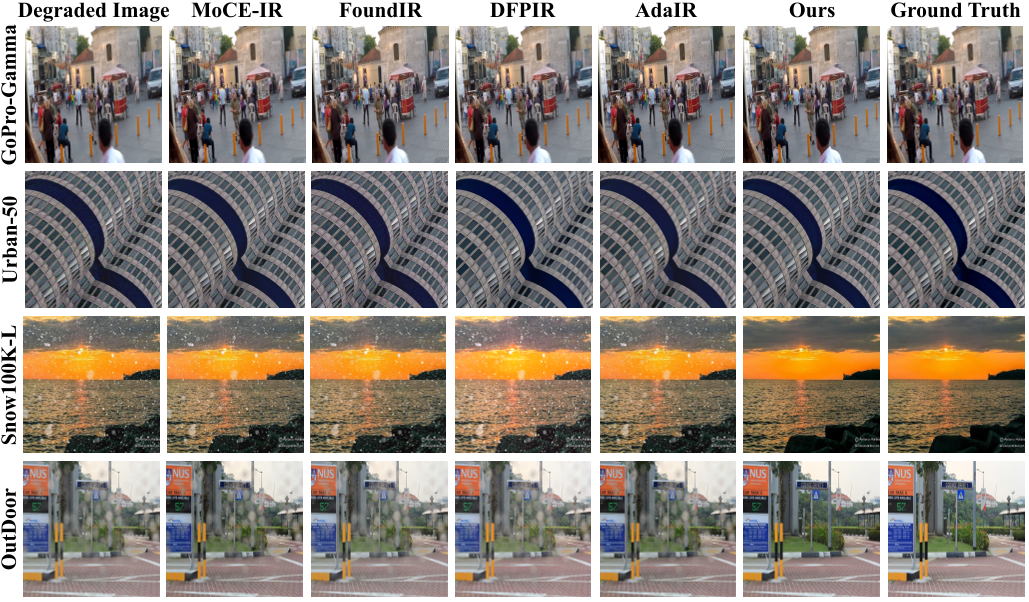}
\caption{High-frequency–dominant tasks (derain/desnow/deblur/denoise): FAPE-IR preserves fine structures and textures with less ringing/oversharpening. Please zoom in for details.}
\label{fig:qual}
\vspace{-4ex}
\end{figure}

As summarized in Table~\ref{tab:IR_series} and Table~\ref{tab:SR_series} (with per-benchmark scores and perceptual metrics detailed in Table~\ref{figs:res} and Table~\ref{figs:res2} in the Appendix), FAPE-IR consistently secures either the best or second-best results across all six AIO-IR task series. On weather-related benchmarks (deraining, desnowing, dehazing), FAPE-IR brings substantial gains over the strongest baselines, typically improving PSNR by about 6–8 dB and reducing perceptual distances (LPIPS, FID, DISTS) by several times (e.g., deraining improves PSNR from 21.94 dB to 28.30 dB and FID from around 100.07 to 21.55). These improvements are most evident on degradations with rich semantic layouts and complex frequency patterns, where the frequency-aware planner guides the LoRA-MoE executor to apply band-specialized experts, yielding cleaner structures and fewer artifacts. For denoising and low-light enhancement, FAPE-IR attains slightly lower PSNR than the best methods but achieves the best SSIM and consistently lower LPIPS/FID/DISTS, suggesting that the combination of frequency-aware planning and adversarial training favors perceptually faithful details. 
On SR benchmarks, FAPE-IR again outperforms all competitors on all five metrics, for instance increasing PSNR from 26.87 dB to 28.53 dB and SSIM from 0.76 to 0.85 while notably reducing FID and DISTS, indicating that our FAPE-IR performs well on mixed degradations. Overall, these results show that explicitly modeling both semantic structure and frequency content allows FAPE-IR to strike a strong balance between distortion and perceptual quality across a broad range of degradations.

To meet real-world latency requirements, Table~\ref{tab:flops-time-transposed} further shows that, despite its larger number of parameters, FAPE-IR maintains reasonable inference time and runs notably faster than the unified counterpart PURE~\cite{wei2025perceive}, suggesting that the framework is feasible for interactive applications.

\subsection{Visual Comparison with Unified Models}
As shown in Figure~\ref{fig:qual4}, we benchmark FAPE-IR against recent unified models on low-level restoration. On RainDrop, most unified models lack explicit low-level modeling and thus fail to complete restoration, either removing only a subset of droplets or causing noticeable color shifts. Emu 3.5~\cite{cui2025emu3} shows a marked jump over earlier Emu variants~\cite{sun2023emu,sun2024generative,wang2024emu3} and contemporaries like BAGEL~\cite{deng2025emerging}, Uniworld-V1~\cite{lin2025uniworld}, and Nexus-Gen~\cite{zhang2504nexus}, likely due to its large corpus and 34.1B parameters. However, its multimodal autoregression design still introduces fine-detail artifacts. By contrast, FAPE-IR reliably restores structures without color drift, underscoring the benefits of frequency-aware planning and band-specialized execution, and pointing toward a path to transfer it to future unified models.

\begin{figure}[t]
\centering
\includegraphics[width=\linewidth]{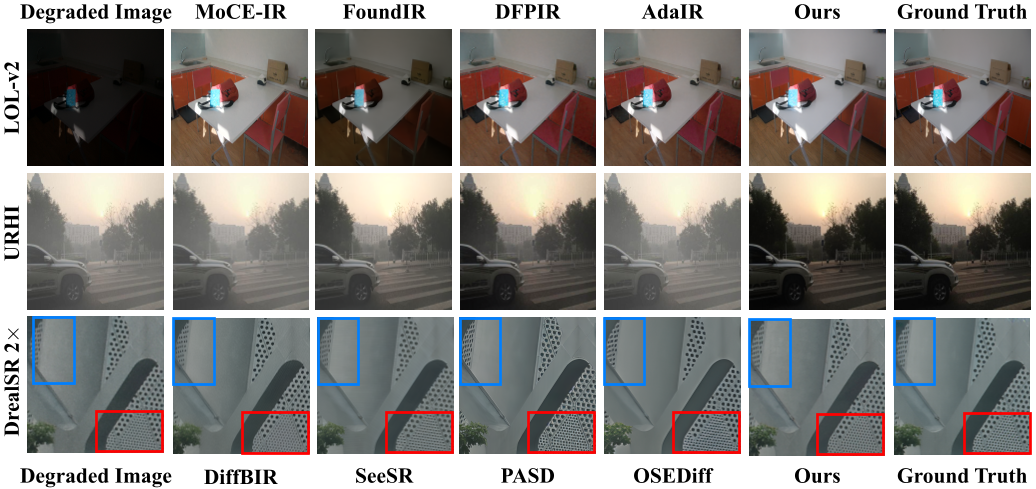}
\caption{Low-frequency \& SR: FAPE-IR cleans haze and balances illumination, preserving color; for SR it yields higher fidelity and fewer artifacts. Please zoom in for details.}
\label{fig:qual2}
\vspace{-3ex}
\end{figure}

\subsection{Visual Comparison with AIO-IR Models}
As shown in Figure~\ref{fig:qual} and Figure~\ref{fig:qual2}, we provide visual comparisons between FAPE-IR and prior AIO-IR methods. Relative to contemporaries, FAPE-IR exhibits stronger visual robustness. In high-frequency tasks, on the large-scale Snow100K-L/S benchmarks, competing methods often fail to recover fine structures, whereas FAPE-IR reconstructs them faithfully. On Urban-50, although a few methods achieve slightly higher scores, FAPE-IR produces visibly less noise and cleaner edges. In low-frequency cases, such as URHI, most baselines restore only portions of the scene, while FAPE-IR more completely removes the haze veil and corrects illumination. For super-resolution, FAPE-IR maintains the highest fidelity to GT with fewer artifacts, preserving textures and text without over-sharpening. These observations are consistent with the quantitative trends. Overall, FAPE-IR suppresses artifacts robustly, maintains strong pixel anchoring, and improves the realism of fine details.

\subsection{Ablation Study}
Under a controlled setting where we fix all hyper-parameters and train for only 10K steps, we perform a minimal ablation of the frequency-aware pipeline. As shown in Table~\ref{tab:moe_r2}, on URHI (low-frequency / haze-dominant), the baseline without Qwen, Freq-U, or Freq-G reaches only $25.03$ dB / $0.92$ SSIM. Simply adding Qwen2.5-VL without any routing control improves performance to $27.95$ dB / $0.94$ SSIM, indicating that semantic planning is beneficial but insufficient to yield effective compute allocation. Coupling planner outputs to the MoE gate (Freq-U) further stabilizes performance at around $28.9$ dB / $0.94$--$0.95$ SSIM, showing that the model can already autonomously perceive and select frequency bands. Injecting token-derived spectral priors into the gate (Freq-G) yields the best result of $29.71$ dB / $0.95$ SSIM, corresponding to a $+4.68$ dB PSNR gain over the no-planner baseline and noticeably more consistent behavior across configurations.

We also study the impact of LoRA--MoE capacity ($r$) and expert allocation. With no frequency priors, a balanced capacity (8/8) is the most stable. Under the Freq-G constraint, a mildly asymmetric allocation (8/16), aligned with the planner-indicated dominant band, maintains near-optimal performance ($29.60$ dB). In contrast, extremely skewed settings (4/16) degrade performance both with and without priors (down to $25.93$ / $24.75$ dB), suggesting that simply reshaping the LoRA structure is insufficient to disentangle frequency bands. Instead, the coupling of frequency priors with semantic planning is key to achieving stable routing and overall performance gains.

\begin{table}[t]
  \centering
  \small
  \caption{Ablation on URHI benchmark. Qwen: remove Qwen2.5-VL; Freq-U: remove frequency-aware text router; Freq-G: remove FIR spectral router; $r$: LoRA rank in expert combos.}
  \label{tab:moe_r2}
  \begin{adjustbox}{scale=0.75}
  \begin{tabular}{cccccccc}
    \toprule
    \textbf{Qwen} & \textbf{Freq-U} & \textbf{Freq-G} & \textbf{r=4} & \textbf{r=8} & \textbf{r=16} &
    \textbf{PSNR$\uparrow$} & \textbf{SSIM$\uparrow$} \\
    \midrule
    \xmark & \xmark & \xmark &  & \cmark &  & 25.03 & 0.92 \\
    \cmark & \xmark & \xmark &  & \cmark &  & 27.95 & 0.94 \\
    \addlinespace[2pt]
    \multicolumn{8}{l}{\textit{+ Freq-U enabled}} \\
    \cmark & \cmark & \xmark &  & \cmark &  & 28.92 & 0.94 \\
    \cmark & \cmark & \xmark &  & \cmark \cmark &  & 28.89 & 0.95 \\
    \cmark & \cmark & \xmark & \cmark &  & \cmark & 24.75 & 0.91 \\
    \cmark & \cmark & \xmark &  & \cmark & \cmark & 24.77 & 0.92 \\
    \addlinespace[2pt]
    \multicolumn{8}{l}{\textit{+ Freq-G enabled}} \\
    \cmark & \cmark & \cmark &  & \cmark \cmark &  & 29.71 & 0.95 \\
    \cmark & \cmark & \cmark & \cmark &  & \cmark & 25.93 & 0.93 \\
    \cmark & \cmark & \cmark &  & \cmark & \cmark & 29.60 & 0.95 \\
    \bottomrule
  \end{tabular}
  \end{adjustbox}
  \vspace{-2ex}
\end{table}

\begin{figure}[t]
\centering
\setlength{\tabcolsep}{4pt}
\begin{subfigure}[t]{0.5\linewidth}
\centering
\includegraphics[width=\linewidth]{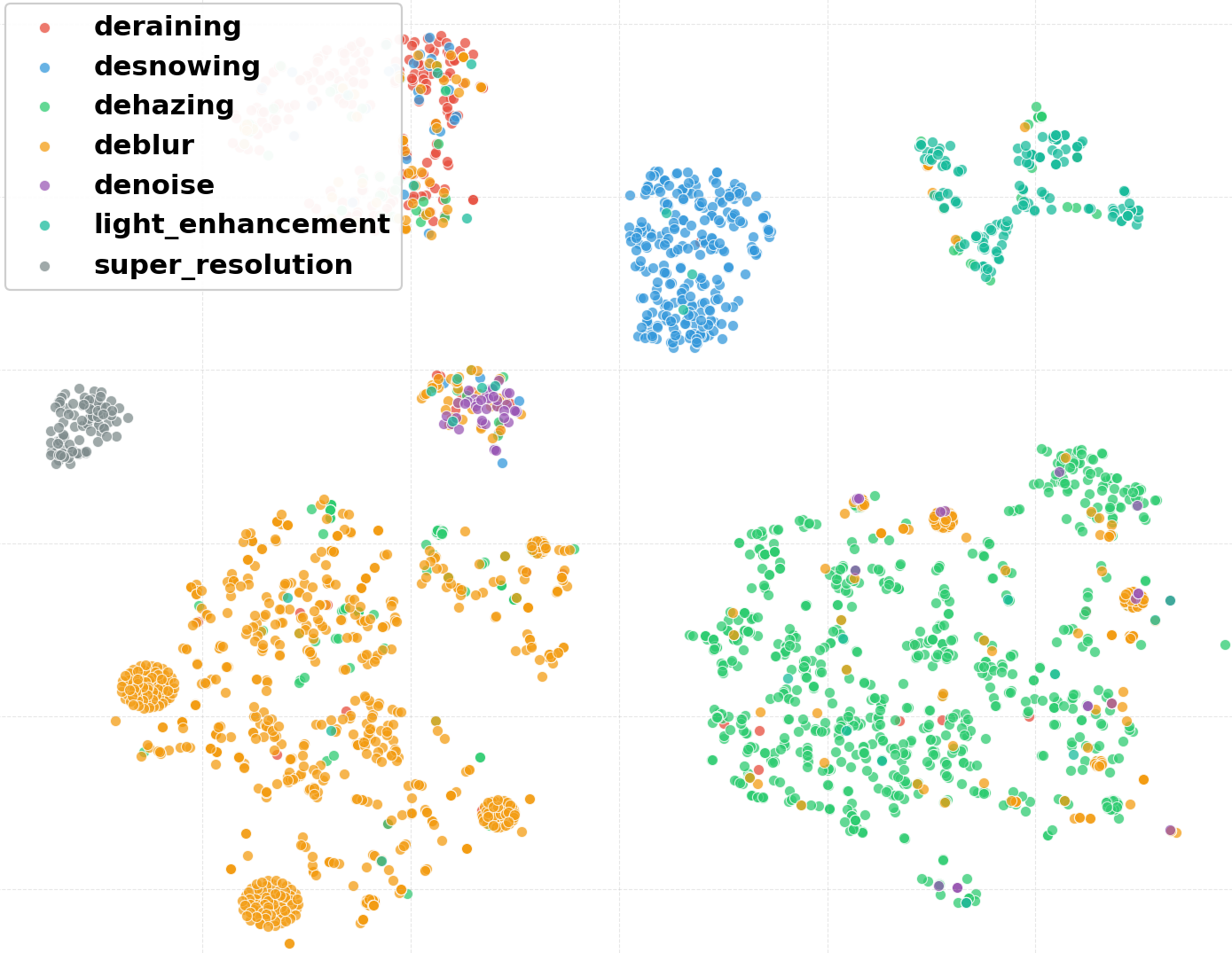}
\label{fig:tsne}
\end{subfigure}\hfill
\begin{subfigure}[t]{0.5\linewidth}
\centering
\includegraphics[width=\linewidth]{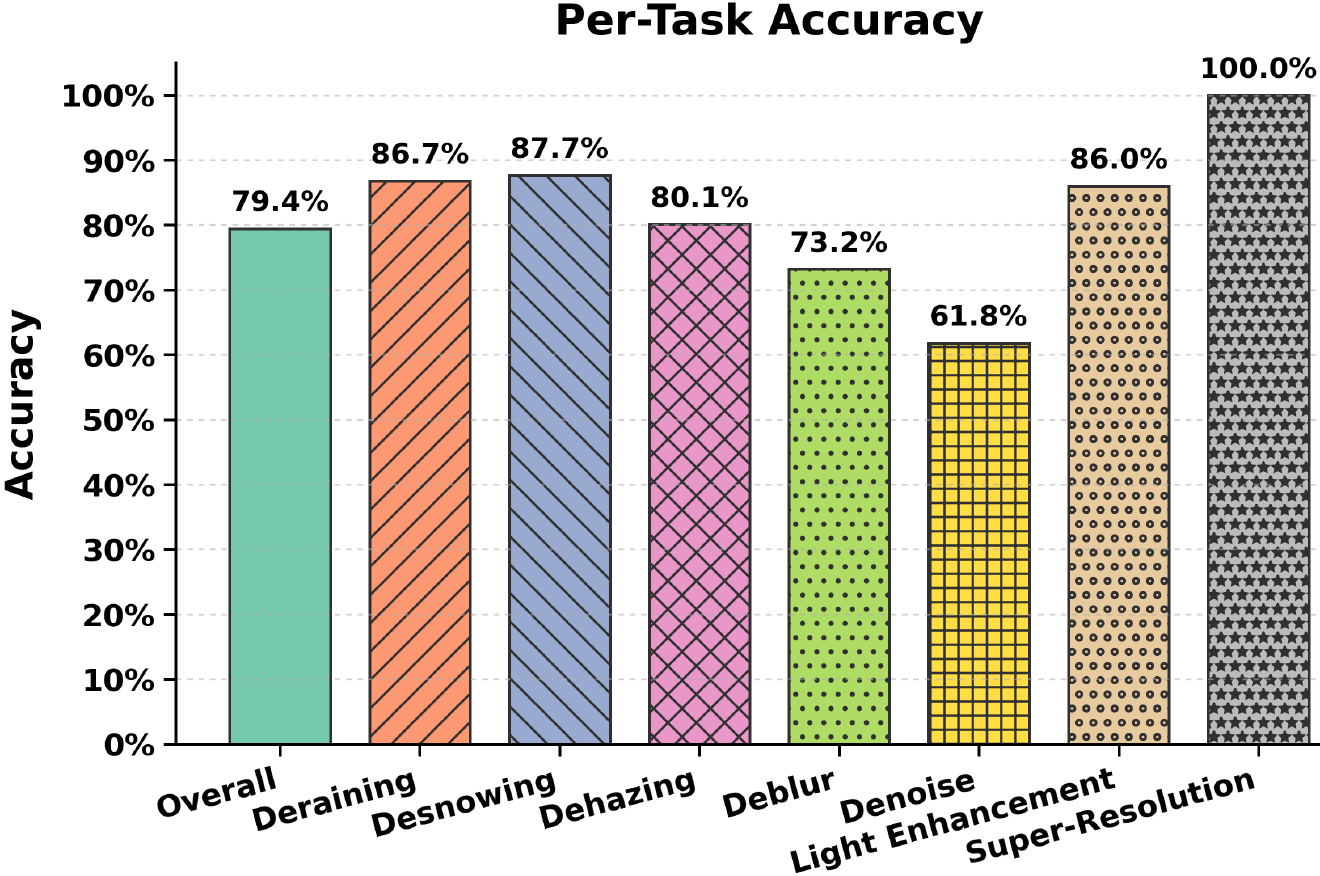}
\label{fig:acc}
\end{subfigure}
\vspace{-0.6em}
\caption{\textbf{Low-level task planning.} (a) The frequency-aware planner yields separable, spectrum-aligned manifolds in feature space using only its own hidden states. (b) Text-based task classification accuracy of Qwen2.5-VL on the corresponding tasks.}
\label{fig:tsne-acc}
\vspace{-3.5ex}
\end{figure}

\subsection{Low-Level Task Planning}
\label{sec:lowlevel-cls}
We assess whether the planner’s decisions form a task- and frequency-aware representation that is both separable and useful. Specifically, we freeze the Qwen2.5-VL planner, extract the decision token sequence $FP$ for each input (BSD68-50, URHI, GoPro, LOL-v2, RainDrop, Snow100K-S, RealSR$\times 2$), and mean-pool their final-layer states. The resulting t-SNE embedding (Figure~\ref{fig:tsne-acc}) exhibits clean, task-dependent manifolds.
Nevertheless, the text-based task readout achieves 79.4\% accuracy overall (all images in BSD68-50 are grayscale, which can spuriously trigger low-light flags).
Moreover, the exemplars (Figure~\ref{fig:cls}) show that the planner’s outputs not only describe but also instantiate a causal frequency control path for restoration.
Collectively, these analyses demonstrate that the planner provides interpretable, frequency-aligned, and causally effective signals that steer band-specialized restoration, explaining the sharper, artifact-averse behavior.

\begin{figure}[t]
\centering
\includegraphics[width=\linewidth]{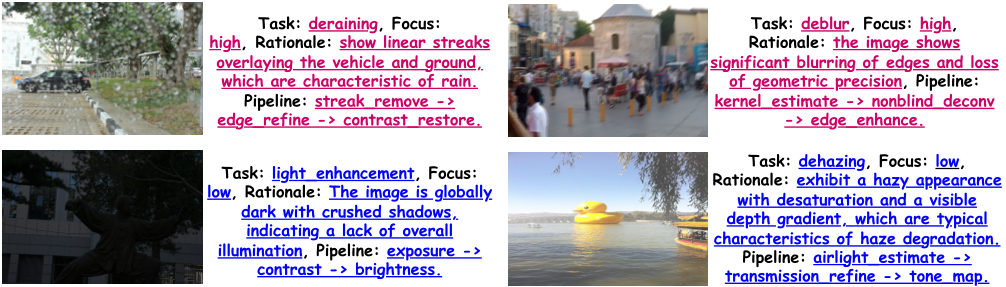}
\caption{\textbf{Frequency-aware planner outputs.} We visualize results on representative corrupted inputs covering four tasks.}
\label{fig:cls}
\vspace{-2ex}
\end{figure}

\begin{figure}[t]
\centering
\includegraphics[width=\linewidth]{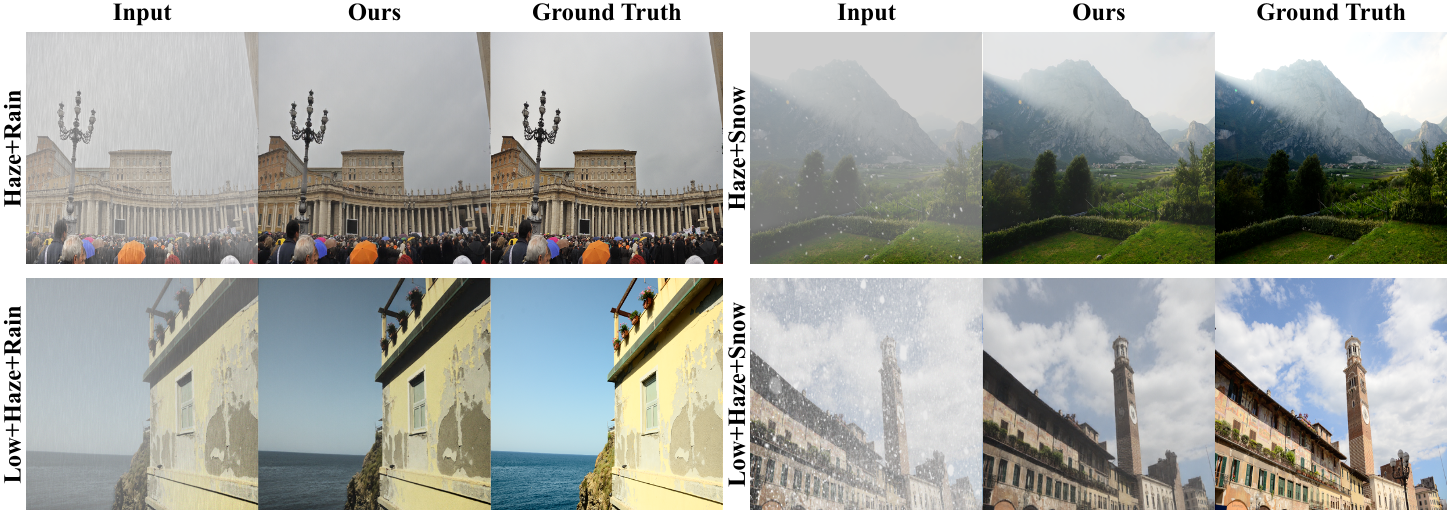}
% \caption{Qualitative generalization to compound degradations (haze+rain/snow; low-light+haze+rain/snow). FAPE-IR removes low-frequency artifacts while preserving high-frequency details and suppressing cross-artifacts. Please zoom in to inspect details.}
\caption{Compound degradations (haze+rain/snow; low-light mixtures): FAPE-IR removes low-frequency artifacts, preserves details, and reduces cross-artifacts. Please zoom in for details.}
\label{fig:qual3}
\vspace{-3ex}
\end{figure}

\subsection{Generalization to Multi-Degradation Tasks}
To assess generalization beyond single degradations, we evaluate FAPE-IR on the recent CDD-11 benchmark~\cite{guo2024onerestore}, which contains real-world images with multi-degradation tasks. As shown in Figure~\ref{fig:qual3}, although FAPE-IR is trained exclusively on single-degradation data, it exhibits strong out-of-distribution performance on these compound mixtures, indicating a non-trivial level of degradation-agnostic generalization. We attribute this to the frequency-aware LoRA-MoE module, which in a single forward pass dispatches complementary experts, allocating low-frequency capacity to illumination correction and veil removal while activating high-frequency capacity around edges and fine textures, consistent with the layered nature of degradations. Overall, the results substantiate FAPE-IR’s robustness and its ability to generalize to unseen composite degradations.

%% file: sec/5_conclusion.tex
\vspace{-1ex}
\section{Conclusion}
We introduce FAPE-IR, a unified AIO-IR framework that couples an interpretable, frequency-aware planner with a diffusion-based executor equipped with the frequency-aware LoRA-MoE module, trained under the adversarial training objective with a single frequency regularizer. By making the semantics–spectrum link explicit, FAPE-IR improves cross-task disentanglement and sharing, stabilizes optimization, and improves open-world robustness. It achieves state-of-the-art or competitive performance across seven AIO-IR tasks and shows strong zero-shot generalization on unseen composite degradations.

\section{Acknowledgements}
This work was supported by National Natural Science Foundation of China (NSFC) under Grants 62441235 and 62176178.

%% file: sec/X_suppl.tex
\clearpage
\setcounter{page}{1}
\maketitlesupplementary

% \setcounter{secnumdepth}{1}
% \appendix

% ---- Section: AR & Flow Matching vs GAN (Low-Level Restoration, Formal Version) ----
\section{Distortions and Over-Generation in AR/FM and Adversarial Training for Low-Level Restoration}
\label{sec:distortion-restore}

\subsection{Notation and Standing Assumptions}
Let $\mathcal{X}=\mathbb{R}^d$ denote image space with Fourier variable $\omega\in\mathbb{R}^d$. Let $p\equiv p_{\mathrm{data}}$ be the ground-truth image law and $q_\theta$ the model law of generator $G_\theta$. For a smoothing schedule $\sigma_t\ge 0$, define $p_t:=p\ast \mathcal{N}(0,\sigma_t^2 I)$ and write $\widehat{f}$ for the Fourier transform. Let $\Phi:\mathcal{X}\to\mathbb{R}^m$ be a fixed perceptual map (LPIPS), frozen during training. Expectations are finite.

We consider a standard linear degradation model for low-level restoration:
\begin{equation}
\label{eq:degradation}
 c = Hx + \eta,\qquad x\sim p,\; \eta\sim \mathcal{N}(0,\sigma_\eta^2 I),
\end{equation}
with known operator $H:\mathcal{X}\to\mathcal{X}$ (e.g., blur+downsample for SR, convolutional blur for deblurring, masking for inpainting). Denote the conditional posterior by $p(x\mid c)$ and its smoothed versions $p_t(x\mid c):=p(\,\cdot\mid c)\ast\mathcal{N}(0,\sigma_t^2 I)$.

\begin{assumption}[Spectral regularity of natural images]
\label{ass:spectral}
The power spectrum of natural images satisfies:
\begin{equation}
    S_{xx}(\omega):=\mathbb{E}|\widehat{x}(\omega)|^2\;\asymp\;\|\omega\|^{-\kappa},\quad \kappa>0,
\end{equation}
so that energy is mainly concentrated at low frequencies and decays polynomially with frequency~\cite{van1996modelling}. This reflects the empirical smoothness of natural images and excludes degenerate high-frequency dominated signals \cite{field1987relations,ruderman1993statistics,hyvarinen2009natural,bora2017compressed,van1996modelling,simoncelli2001natural}.
\end{assumption}

\begin{assumption}[Forward operator]
\label{ass:H}
$H$ is linear, bounded, and shift-invariant on the domain of interest with frequency response $\widehat{H}(\omega)$. For blur/downsample/inpainting, $|\widehat{H}(\omega)|$ decays (or vanishes) for large $\|\omega\|$ or outside the passband~\cite{ulyanov2018deep,krishnan2009fast,levin2009understanding}.
\end{assumption}

\begin{assumption}[Regularized flow path]
\label{ass:path}
The conditional path is given by Gaussian smoothing~\cite{sporring2013gaussian,witkin1987scale,lindeberg1998feature}:
\begin{equation}
    p_t(\cdot\mid c)=p(\cdot\mid c)\ast K_{\sigma_t},
\end{equation}
with $\sigma_t$ smooth in $t$. Then $p_t$ is $C^1$ in $t$, and there exists a drift $v_t^\star(\cdot\mid c)$ such that $\partial_t p_t+\nabla\!\cdot(v_t^\star p_t)=0$. For $p_t=p\ast\mathcal N(0,\sigma_t^2 I)$ one can take the explicit choice~\cite{romano2017little,venkatakrishnan2013plug}:
\begin{equation}
    v_t^\star(x,c)= -\tfrac{\dot\sigma_t^2}{2}\,\nabla\log p_t(x\mid c) .
\end{equation}
\end{assumption}

\begin{assumption}[Lipschitz drift and locally bounded conditional density]
\label{ass:lipschitz}
Each $v_t^\star(\cdot\mid c)$ is $L$-Lipschitz. Moreover, for each $t$ there exists a compact set $K_t\subset\mathcal X$ such that the conditional densities are locally bounded and bounded away from zero on $K_t$:
\begin{equation}
    0<m_t\le p_t(x\mid c)\le M_t<\infty\quad \text{for all }x\in K_t,
\end{equation}
uniformly (in the sense appropriate to the expectations we consider). All expectations in this paper are taken over $x$ restricted to $K_t$ (or via truncation/limit arguments), which is natural in low-level restoration where reconstructions lie in a plausible compact region.
\end{assumption}

\begin{assumption}[IPM critic: attainment and compact parameterization]
\label{ass:ipm}
We assume $d_{\mathcal F}(p,q)=\sup_{f\in\mathcal F}(\mathbb E_p f-\mathbb E_q f)$ with $\mathcal F=\{f_\psi:\psi\in\Psi\}$ parameterized by a compact set $\Psi$, such that the supremum is attained at some $\psi^\star\in\Psi$. Then Danskin’s theorem applies to yield:
\begin{equation}
    \nabla_\theta d_{\mathcal F}(p,q_\theta) \;=\; -\,\mathbb E\!\left[J_{G_\theta}^\top \nabla_x f_{\psi^\star}(G_\theta)\right].
\end{equation}
(Alternatively, one may work with subgradients if $\mathcal F$ is non-compact, e.g.\ the $1$-Lipschitz ball for $W_1$.) See \cite{sriperumbudur2009integral,arjovsky2017wasserstein,ambrosio2021lectures}.
\end{assumption}

\subsection{Objectives (Restoration Setting)}
\noindent\textbf{Autoregression (AR).} Conditioned on $c$, the chain factorization is $q_\theta(x\mid c)=\prod_{t=1}^T q_\theta(x_t\mid x_{<t},c)$, and
\begin{equation}
\label{eq:ar-ml-cond}
\theta^\star_{\mathrm{AR}}\in\arg\min_\theta\,\mathbb{E}_{c\sim p(c)}\,\mathrm{KL}\big(p(\cdot\mid c)\,\|\,q_\theta(\cdot\mid c)\big).
\end{equation}

\noindent\textbf{Flow Matching (FM).} 
FM is a general framework for distribution transport, with Rectified Flow as the linear-path special case; related score-based views connect to diffusion/SDE modeling and score matching \cite{lipman2022flow,liu2022flow,song2020score,hyvarinen2005estimation}. Here we adopt the basic form based on conditional scores of smoothed posteriors:
\begin{equation}
\label{eq:fm-obj}
    \begin{aligned}
        \theta^\star_{\mathrm{FM}}\in \arg\min_{\theta}\int_0^1\! w(t)\,
        \mathbb{E}_{c\sim p(c)}\,\mathbb{E}_{x_t\sim p_t(\cdot\mid c)} \\
        \big\|s_\theta(x_t,c,t)-\nabla\log p_t(x_t\mid c)\big\|_2^2\,dt.
    \end{aligned}
\end{equation}

\noindent\textbf{Adversarial Training.} Given pairs $(c,x)$, augment the composite objective with a measurement term:
\begin{equation}
\label{eq:composite-restore}
\begin{aligned}
\mathcal{J}(\theta)
&=\lambda\, d_{\mathcal{F}}\big(p, q_\theta\big)
+\alpha\,\mathbb{E}\,\|G_\theta(c)-x\|_2^2 \\
&\quad+\beta\,\mathbb{E}\,\|\Phi(G_\theta(c)) - \Phi(x)\|_2^2,
% &\quad+\gamma\,\mathbb{E}\,\|H G_\theta(c) - c\|_2^2,
\quad \lambda,\alpha,\beta\ge 0,
\end{aligned}
\end{equation}
with $d_{\mathcal F}$ often chosen as $W_1$ (WGAN) \cite{arjovsky2017wasserstein,sriperumbudur2009integral,ambrosio2021lectures}.

\subsection{AR Causes Artifacts and Over-Generation}
\label{sec:ar-artifacts}

\noindent\textbf{Modeling assumptions for AR proofs.}
We make explicit the family of admissible AR conditionals and the inference rule used at test time.

\begin{assumption}[Single-mode AR conditionals or deterministic limit]\label{ass:ar-family}
For each $t$, conditionals $q_\theta(x_t\mid x_{<t},c)$ belong to a centered, symmetric, log-concave location family:
\begin{equation}
    \mathcal Q:=\big\{\,x_t\mapsto f_\sigma(x_t-\mu_t)\;:\; \mu_t\in\mathbb{R}^{d_t},\; \sigma\in(0,\sigma_{\max}]\,\big\},
\end{equation}
where $f_\sigma$ is even, strictly log-concave, smooth in $\sigma$, and the \emph{deterministic limit} $\sigma\!\to\!0$, i.e.\ $q_\theta(\cdot\mid x_{<t},c)\Rightarrow \delta_{\mu_t(x_{<t},c)}$.
\end{assumption}

\begin{assumption}[Greedy or vanishing-temperature decoding]\label{ass:greedy}
At test time, decoding is either greedy ($x_t=\arg\max q_\theta(\cdot\mid x_{<t},c)$) or stochastic with temperature $\tau\downarrow 0$ so that $x_t\Rightarrow \arg\max q_\theta(\cdot\mid x_{<t},c)$ in probability.
\end{assumption}

\begin{definition}[Posterior $\ker(H)$-ambiguity at step $t$]
\label{def:ambig}
Let $E_t:\mathbb{R}^{d_t}\to\mathcal X$ be the (fixed) embedding that maps the local step-$t$ variable into image space. We say that the step-$t$ posterior exhibits a symmetric $\ker(H)$-ambiguity with gap $u_t\in\mathbb{R}^{d_t}$ if for a measurable set $A_t\subseteq\{(x_{<t},c)\}$ with $p(A_t)\ge \rho_t>0$,
\begin{equation}
    \begin{aligned}
    p(x_t \mid x_{<t},c)
    &= \tfrac12 \,\delta_{x_t^{(0)}(x_{<t},c)+u_t} \\
    & + \tfrac12 \,\delta_{x_t^{(0)}(x_{<t},c)-u_t},\\
    H\,E_t u_t &= 0 .
    \end{aligned}
\end{equation}
Here $x_t^{(0)}$ is the posterior midpoint in the local coordinates; the corresponding image-space ambiguity direction is $E_t u_t$.
\end{definition}

\begin{theorem}[AR deterministic collapse along $\ker(H)$]\label{thm:ar-collapse}
Under Assumptions~\ref{ass:ar-family}--\ref{ass:greedy} and Definition~\ref{def:ambig}, the maximum-likelihood solution of Equation~\ref{eq:ar-ml-cond} satisfies
\begin{equation}
   \mu_t^\star(x_{<t},c)=x_t^{(0)}(x_{<t},c) \text{ for $p$-a.e.\ $(x_{<t},c)\in A_t$}, 
\end{equation}
and the decoded token obeys $x_t=\mu_t^\star(x_{<t},c)$ a.s.\ in the greedy limit. Consequently, the image-space component of $x_t$ along the ambiguity direction $\operatorname{span}\{E_t u_t\}$ is collapsed to zero mean, removing genuine $\pm E_t u_t$ variability and inducing structured artifacts aligned with $\ker(H)$.
\end{theorem}

\begin{proof}
Fix $(x_{<t},c)\in A_t$. By Assumption~\ref{ass:ar-family} the conditional log-likelihood for center $\mu$ is
\(
\ell(\mu)=\log f_\sigma(x_t^{(0)}+u_t-\mu)+\log f_\sigma(x_t^{(0)}-u_t-\mu).
\)
Since $f_\sigma$ is even and strictly log-concave, $\ell$ is strictly concave and achieves its unique maximum at the midpoint $\mu=x_t^{(0)}$. Teacher forcing MLE therefore yields $\mu_t^\star=x_t^{(0)}$ on $A_t$, and by Assumption~\ref{ass:greedy} the decoded value equals $\mu_t^\star$ almost surely in the deterministic limit. Because $H E_t u_t=0$, the collapsed direction is exactly unobservable in $c$, hence constitutes artifact-prone freedom. 
\end{proof}

\begin{remark}[When collapse can be avoided]\label{rem:avoid}
If the conditional family admits \emph{explicit multimodality} (e.g.\ mixture components with separate centers and nontrivial sampling at test time), Theorem~\ref{thm:ar-collapse} need not hold. This motivates latent-variable AR or distributional heads when $H$ is severely rank-deficient.
\end{remark}

\noindent\textbf{From local collapse to global artifact patterns.}
We connect stepwise ambiguity to image-level artifacts via an additive geometric functional.

\begin{definition}[Artifact \emph{deficit}]
\label{def:artifact}
Let $\Pi:\mathcal X\to\mathbb{R}_{\ge 0}$ be a convex seminorm measuring nullspace-structured deviations (e.g.\ $\Pi(x)=\|P_{\ker(H)}x\|_2$ or a high-pass seminorm restricted to $\ker(H)$). For a sample $\tilde x\sim q_\theta(\cdot\mid c)$ define the \emph{artifact deficit}:
\begin{equation}
    \mathcal D(\tilde x;c)\;:=\;\mathbb E_{p(x\mid c)}\Pi(x)\;-\;\Pi(\tilde x).
\end{equation}
By Jensen’s inequality and Theorem~\ref{thm:ar-collapse}, when the posterior on a direction is symmetric and decoding collapses to the midpoint, $\mathcal D$ is typically nonnegative; its positivity quantifies the loss of genuine posterior variability along $\ker(H)$.
\end{definition}

\begin{proposition}[TV-based upper bound on deficit over ambiguity sets]
\label{prop:artifact-upper-fixed}
Suppose Definition~\ref{def:ambig} holds for indices $t\in\mathcal T\subseteq\{1,\dots,T\}$ with probabilities $\{\rho_t\}$ and ambiguity vectors $\{u_t\}$. Let $\Pi$ be any convex seminorm that is strictly convex on $\ker(H)$. For each $R>0$, define the truncated seminorm $\Pi_R:=\min\{\Pi,R\}$ to ensure boundedness on $K_t$ (Assumption~\ref{ass:lipschitz}). Define on $A_t$:
\begin{equation}
    \bar\varepsilon_t^{(A)}
    :=\operatorname*{ess\,sup}_{(x_{<t},c)\in A_t}
    \mathrm{TV}\!\big(p(\cdot\mid x_{<t},c),\,q_\theta(\cdot\mid x_{<t},c)\big).
\end{equation}
Let
\begin{equation}
\begin{aligned}
&\Delta_{\Pi_R}^{\max} 
:= \sup_{t\le T}\sup_{(x_{<t},c)\in A_t} \tfrac12\Bigl(\Pi_R(x_t^{(0)}{+}u_t) \\
&+\Pi_R(x_t^{(0)}{-}u_t)-2\Pi_R(x_t^{(0)})\Bigr).
\end{aligned}
\end{equation}
Then
\begin{equation}
\label{eq:artifact-upper-correct}
\begin{aligned}
&\mathbb{E}\bigl[\mathcal{D}_R(\tilde x;c)\bigr]
\ \le\ 
\Delta_{\Pi_R}^{\max}\sum_{t=1}^{T}\rho_t\,\bar\varepsilon_t^{(A)} \\
&\;+\;\mathbb E\!\big[\mathcal D_R(\tilde x;c)\,\mathbf 1_{(\cup_t A_t)^\complement}\big],
\end{aligned}
\end{equation}
where $\mathcal D_R$ is defined by replacing $\Pi$ with $\Pi_R$ in Definition~\ref{def:artifact}. Passing to the limit $R\to\infty$ by monotone convergence yields the same bound for $\mathcal D$ provided the moments in Assumption~\ref{ass:lipschitz} hold. The complement term can be further bounded once additional structure outside $\cup_t A_t$ is specified (e.g.\ no nullspace ambiguity or separate curvature bounds).
\end{proposition}

\begin{proof}
On each $A_t$, for bounded $\Pi_R$ we have $|\mathbb E_p \Pi_R - \mathbb E_q \Pi_R|\le \|\Pi_R\|_\infty \,\mathrm{TV}(p,q)$. Summing over $t\in\mathcal T$ and using the definition of $\Delta_{\Pi_R}^{\max}$ yields Equation~\ref{eq:artifact-upper-correct}. The monotone limit follows by Assumption~\ref{ass:lipschitz}.
\end{proof}

\noindent\textbf{Quantifying Exposure Amplification in Rollout.}

\begin{assumption}[Lipschitz artifact seminorm] \label{ass:Pi}
There exists $L_\Pi>0$ such that $|\Pi(x)-\Pi(y)|\le L_\Pi \|P_{\ker(H)}(x-y)\|_2$ for all $x,y$.
When needed, we also use the truncated seminorm $\Pi_R:=\min\{\Pi,R\}$ to guarantee integrability; all bounds pass to the limit $R\to\infty$ by monotone convergence under our moment assumptions.
\end{assumption}

\begin{proposition}[Positive artifact deficit under repeated ambiguity]
\label{prop:artifact-deficit}
Suppose Definition~\ref{def:ambig} holds for indices $t\in\mathcal T\subseteq\{1,\dots,T\}$ with probabilities $\{\rho_t\}$ and ambiguity vectors $\{u_t\}$. Let $\Pi$ be any convex seminorm that is strictly convex on $\ker(H)$. Then:
\begin{equation}
\begin{split}
&\mathbb E\big[\mathcal D(\tilde x;c)\big]\;\ge\;
\sum_{t\in\mathcal T}\rho_t\Bigl[\tfrac12\Pi\!\big(x_t^{(0)}{+}u_t\big) \\ 
&+\tfrac12\Pi\!\big(x_t^{(0)}{-}u_t\big)\Bigr]
-\sum_{t\in\mathcal T}\rho_t\,\Pi\!\big(x_t^{(0)}\big).
\end{split}
\end{equation}
\end{proposition}

\begin{proof}
Fix $(x_{<t},c)\in A_t$. Under $p$, $\Pi$ sees two symmetric values around $x_t^{(0)}$; by convexity $\tfrac12\Pi(x_t^{(0)}{+}u_t)+\tfrac12\Pi(x_t^{(0)}{-}u_t)\ge \Pi(x_t^{(0)})$, with strict inequality by strict convexity if $u_t\ne 0$. Under $q_\theta$, Theorem~\ref{thm:ar-collapse} yields the degenerate value $\Pi(x_t^{(0)})$. Taking differences and averaging over $A_t$ gives the bound; summing over indices $t\in\mathcal T$ yields the stated inequality. A matching lower/upper bound in general needs additional independence/orthogonality assumptions and is omitted here.
\end{proof}

\noindent\textbf{Exposure bias and accumulation.}
During training, loss is minimized with teacher forcing; at test time the model uses its own predictions. Define the per-step deviation:
\begin{equation}
\begin{split}
\delta_t(x_{<t},c) &:= \mathrm{TV}\!\big(p(\cdot\mid x_{<t},c),\,q_\theta(\cdot\mid x_{<t},c)\big), \\
\varepsilon_t &:= \mathbb{E}_{x_{<t}\sim p(\cdot\mid c)}\,\delta_t(x_{<t},c).
\end{split}
\end{equation}
We also write the \emph{worst-case} per-step deviation $\bar\varepsilon_t:=\operatorname*{ess\,sup}_{x_{<t},c}\delta_t(x_{<t},c)$.

\begin{lemma}[Accumulation of local deviations]
\label{lem:exposure-final}
For the AR factorizations of $p$ and $q_\theta$ using consistent measurable versions of the conditionals, the following bounds hold:
\begin{equation}
\label{eq:tv-bound}
\mathrm{TV}\!\big(p(\cdot\mid c),q_\theta(\cdot\mid c)\big)
\;\le\; 1-\prod_{t=1}^T (1-\bar\varepsilon_t)
\;\le\; \sum_{t=1}^T \bar\varepsilon_t,
\end{equation}
and, independently,
\begin{equation}
\mathrm{TV}\!\big(p(\cdot\mid c),q_\theta(\cdot\mid c)\big)
\;\le\; \sum_{t=1}^T \varepsilon_t.
\end{equation}
In particular, for small $\{\alpha_t\}$ with $\alpha_t\in\{\,\bar\varepsilon_t\,\}$,
\begin{equation}
\label{eq:tv-expansion}
1-\prod_{t=1}^T (1-\alpha_t)
= \sum_{t=1}^T \alpha_t \;+\; \mathcal O\!\Big(\sum_{s<t}\alpha_s\alpha_t\Big).
\end{equation}
\end{lemma}

\noindent\textbf{Chain rule and conditional optimization.}
By standard decomposition, 
\begin{equation}
\label{eq:kl-decomp}
\begin{split}
&\mathrm{KL}\!\big(p(\cdot\mid c)\,\|\,q_\theta(\cdot\mid c)\big)
= \sum_{t=1}^{T}\,\mathbb{E}_{x\sim p(\cdot\mid c)} \\
&\quad \Big[\mathrm{KL}\big(p(\cdot\mid x_{<t},c)\,\|\,q_\theta(\cdot\mid x_{<t},c)\big)\Big].
\end{split}
\end{equation}
Thus training matches each conditional distribution separately. 
However, if $p(x_t\mid x_{<t},c)$ admits multiple plausible continuations
(e.g.\ along weakly observed directions of $H$), then 
any surrogate that collapses to a single conditional predictor necessarily \emph{destroys} these alternatives. 
This mismatch does not show up as blur in practice, samples remain sharp, but it manifests as systematic artifacts where the model enforces spurious ``averaged" structures.

\begin{proof}[Proof of Equation~\ref{eq:kl-decomp}]
Write $p(x\mid c)=\prod_{t=1}^T p(x_t\mid x_{<t},c)$ and $q_\theta(x\mid c)=\prod_{t=1}^T q_\theta(x_t\mid x_{<t},c)$. Then
\begin{equation}
\small
\begin{aligned}
    &\mathrm{KL}\!\big(p\|q_\theta\big)
    =\mathbb E_{x\sim p}\!\left[\sum_{t=1}^T\log\frac{p(x_t\mid x_{<t},c)}{q_\theta(x_t\mid x_{<t},c)}\right]\\
    &=\sum_{t=1}^T \mathbb E_{x_{<t}\sim p}\,
    \mathrm{KL}\!\big(p(\cdot\mid x_{<t},c)\,\|\,q_\theta(\cdot\mid x_{<t},c)\big).
\end{aligned}
\end{equation}
\noindent Hence Equation~\ref{eq:kl-decomp} holds, formalizing that AR training aligns each conditional factor separately.
\end{proof}

\begin{proposition}[Nullspace ambiguity induces artificial averaging]
\label{prop:nullspace-final}
Let $H$ in Equation~\ref{eq:degradation} have nontrivial nullspace and suppose the posterior $p(x\mid c)$ 
assigns equal mass to $x_\pm=x_0\pm u$ with $Hu=0$.
Then the Bayes point estimator is:
\begin{equation}
    \widehat x=\mathbb E[x\mid c]=x_0,
\end{equation}
which \emph{cancels} the true variability $\pm u$.
Thus any deterministic AR predictor removes detail along $u$, creating structured artifacts not present in the true posterior (while stochastic decoding that breaks symmetry may \emph{inconsistently} realize one of the modes).
\end{proposition}

\begin{proof}
Under squared loss, the Bayes estimator equals the conditional mean:
\begin{equation}
    \widehat x=\mathbb E[x\mid c]=\tfrac12(x_0+u)+\tfrac12(x_0-u)=x_0.
\end{equation}
Hence the genuine posterior variability $\pm u$ (unobservable since $Hu=0$) is averaged out.

Moreover, consider an AR conditional family restricted to a symmetric log-concave location family
(e.g.\ $x_t\!\sim\!\mathcal N(\mu_t,\sigma^2 I)$ with fixed $\sigma$, or the $\sigma\!\to\!0$ deterministic limit).
Given a symmetric two-point posterior $\tfrac12\delta_{x_0+u}+\tfrac12\delta_{x_0-u}$, the conditional MLE for the center is the midpoint $x_0$ by symmetry and concavity:
the log-likelihood takes the form $\ell(\mu)\!=\!\log f(x_0+u-\mu)+\log f(x_0-u-\mu)$ with $f$ log-concave and even; $\ell$ is concave and maximized at $\mu=x_0$ (where the two arguments are opposite), yielding the same collapse to $x_0$. 
Thus single-mode/deterministic conditionals enforce artificial averaging along nullspace directions.
\end{proof}

\noindent\textbf{Takeaways.}
AR provably enforces \emph{deterministic resolutions} of measurement ambiguities, 
averages away legitimate alternatives along nullspace directions,
and accumulates local deviations into global structural errors. 
Unlike blur, which is not observed in practice, these mechanisms explain 
the clear yet artifact-laden outputs: geometric kinks, inconsistent fills, 
and over-generation of non-existent structures. Exposure/rollout issues are closely related to scheduled sampling, sequence-level training, and imitation-learning reductions \cite{bengio2015scheduled,ranzato2015sequence,ross2011reduction}.

\noindent\textbf{Remark on bounds.}
The product-form term in Equation~\ref{eq:tv-bound} uses worst-case per-step deviations $\bar\varepsilon_t$ and is tight for sequential maximal couplings~\cite{ulyanov2018deep}.
A nontrivial \emph{lower} bound on $\mathrm{TV}(p,q_\theta)$ or on artifact deficits in terms of per-step quantities generally requires additional independence/decoupling assumptions; without them we provide only the robust upper bounds above (cf.\ Proposition~\ref{prop:artifact-upper-fixed}).

\subsection{Why Distortions and Over-generation Appear in Flow Matching Optimization}
\label{sec:fm-distort}

\noindent\textbf{A precise local sandwich.}
We refine the conditional spectral sandwich (Theorem~\ref{thm:spectral-gap}, stated later) by (i) focusing on the small-perturbation regime and (ii) replacing $|\widehat H(\omega)|^2$ with an \emph{information transfer coefficient} $\Xi_t(\omega)$ capturing the $H$-noise-prior interplay~\cite{ulyanov2018deep}.

\begin{assumption}[Small-perturbation regime and bounded densities]\label{ass:small-delta}
Let $p_t(\cdot\mid c)=p(\cdot\mid c)\ast K_{\sigma_t}$ and assume that $q_{\theta,t}(\cdot\mid c)$ admits the representation:
\begin{equation}
q_{\theta,t}(\cdot\mid c)=q_\theta(\cdot\mid c)\ast K_{\sigma_t},
p_t(\cdot\mid c)=p(\cdot\mid c)\ast K_{\sigma_t}.
\end{equation}
Define the \emph{unsmoothed} discrepancy $\delta(\cdot\mid c):=q_\theta(\cdot\mid c)-p(\cdot\mid c)$ and its smoothed version
\begin{equation}
\Delta_t(\cdot\mid c):=\delta(\cdot\mid c)\ast K_{\sigma_t},
\end{equation}
so that $q_{\theta,t}=p_t+\Delta_t$ and $\int \Delta_t(x\mid c)\,dx=0$.
Assume $\|\Delta_t\|_{L^\infty}\le \eta_t m_t/2$ with some $0<\eta_t<1$ and $\|\nabla\Delta_t\|_{L^2}<\infty$, where $0<m_t\le p_t\le M_t<\infty$ are as in Assumption~\ref{ass:lipschitz}.
\end{assumption}

\begin{definition}[Information transfer coefficient]\label{def:Xi}
In the linear–Gaussian setting with $x\sim\mathcal N(0, S_{xx})$ and \(c=Hx+\eta\), define the frequency-wise \emph{posterior information weight}
\begin{equation}
\Xi_{\mathrm{LG}}(\omega) \;:=\; \frac{|\,\widehat H(\omega)\,|^2}{\sigma_\eta^2 + |\,\widehat H(\omega)\,|^2 S_{xx}(\omega)}\in[0,1/\sigma_\eta^2]\!,
\end{equation}
which vanishes where $\widehat H(\omega)=0$ and increases with the local posterior SNR. 
In the general small-perturbation regime of Assumption~\ref{ass:small-delta}, there exist constants $0<c_t^{(1)}\le c_t^{(2)}<\infty$ (depending only on $(m_t,M_t)$ and low-order moments of $p_t(\cdot\mid c)$, but not on $\omega$) such that:
\begin{equation}
    c_t^{(1)}\,\Xi_{\mathrm{LG}}(\omega)
    \;\le\; \Xi_t(\omega)
    \;\le\; c_t^{(2)}\,\Xi_{\mathrm{LG}}(\omega).
\end{equation}
\end{definition}

\begin{lemma}[Pythagorean decomposition in $L^2(p_t)$]\label{lem:proj}
Let $s_\theta=\nabla\log q_{\theta,t}+r_{\theta,t}$ with $r_{\theta,t}\in L^2(p_t)$. Then
\begin{equation}
\begin{split}
  &\mathbb E_{p_t}\|s_\theta-\nabla\log p_t\|_2^2
  = D_F(p_t\|q_{\theta,t}) + \|r_{\theta,t}\|_{L^2(p_t)}^2 \\[-1mm]
  &\quad+ 2\langle r_{\theta,t},\,\nabla\log q_{\theta,t}-\nabla\log p_t\rangle_{L^2(p_t)} .
\end{split}
\end{equation}
If $r_{\theta,t}$ is the $L^2(p_t)$-orthogonal residual of projecting $\nabla\log p_t$ onto the model class, the cross term vanishes and the training loss splits into a Fisher part plus an approximation error (this orthogonality is an idealized projection property and need not hold for general parameterizations).
\end{lemma}

\begin{theorem}[Local Fisher sandwich with spectral weights]\label{thm:local-sandwich}
Under Assumptions~\ref{ass:H}, \ref{ass:lipschitz}, and \ref{ass:small-delta}, there exist finite constants $0<c_t\le C_t<\infty$ such that the (conditional) Fisher divergence satisfies:
\begin{equation}\label{eq:local-sandwich}
\begin{aligned}
&\int_0^1\! w(t)\,c_t \!\int_{\mathbb{R}^d} 
   \underbrace{\|\omega\|^2 e^{-\sigma_t^2\|\omega\|^2}\,\Xi_t(\omega)}_{=:\widetilde W_t(\omega)}
   \,\mathbb E_c\big|\widehat{\delta}(\omega\mid c)\big|^2\,d\omega\,dt \\[2mm]
&\le\; \int_0^1\! w(t)\,D_F\!\big(p_t(\cdot\mid c)\,\|\,q_{\theta,t}(\cdot\mid c)\big)\,dt \\[1mm]
&\le\;
   \int_0^1\! w(t)\,C_t \!\int_{\mathbb{R}^d}\!
   \widetilde W_t(\omega)\,\mathbb E_c\big|\widehat{\delta}(\omega\mid c)\big|^2\,d\omega\,dt + R ,
\end{aligned}
\end{equation}
with a quadratic remainder
\(
0\le R \le \int_0^1 w(t)\,\kappa_t\,\|\Delta_t\|_{L^\infty}^2 \,dt,
\)
where $\kappa_t$ depends on $(m_t,M_t)$ and the Lipschitz constants in Assumption~\ref{ass:lipschitz}. 
As $\max_t\|\Delta_t\|_{L^\infty}\!\to\!0$, we have $R\!\to\!0$, 
and Equation~\ref{eq:local-sandwich} becomes an equality up to factors $c_t,C_t$.
Moreover, on any time window where $\sigma_t\in[\sigma_{\min},\sigma_{\max}]$ is bounded, the constants $c_t,C_t$ can be chosen uniformly bounded with respect to $t$ on that window.
\end{theorem}

\begin{corollary}[From Fisher to the FM objective]
Under the setting of Lemma~\ref{lem:proj}, the FM loss in Equation~\ref{eq:fm-obj} decomposes into the Fisher term bounded by Equation~\ref{eq:local-sandwich} plus a nonnegative approximation error $\|r_{\theta,t}\|^2_{L^2(p_t)}$ (when the cross term vanishes). Hence all conclusions drawn from Equation~\ref{eq:local-sandwich} for $D_F$ transfer to $\mathcal L_{\mathrm{FM}}$ up to adding this nonnegative term.
\end{corollary}

\begin{corollary}[High-frequency underweighting and nullspace gaps]\label{cor:rigor-bias}
Under Theorem~\ref{thm:local-sandwich}, for any set $\Omega\subset\mathbb{R}^d$:
\begin{enumerate}
\item If $\inf_{\omega\in\Omega}\|\omega\|\to\infty$, then $\sup_{\omega\in\Omega} \widetilde W_t(\omega)\to 0$ exponentially in $\|\omega\|$, hence discrepancies concentrated in $\Omega$ contribute arbitrarily little to $\mathcal L_{\mathrm{FM}}$.
\item If $|\widehat H(\omega)|=0$ on $\Omega$, then $\Xi_t(\omega)=0$ on $\Omega$ (for the linear–Gaussian definition) and thus $\widetilde W_t(\omega)=0$; this establishes exact blindness on $\ker(H)$ bands. When $|\widehat H(\omega)|$ is nonzero but very small, the weight remains near-blind.
\end{enumerate}
\end{corollary}

\noindent\textbf{Early-time forgetting bound.}
\begin{proposition}[Quantified forgetting under strong smoothing (expectation version)]\label{prop:forgetting}
Let $p_t(\cdot\mid c)=p(\cdot\mid c)\ast \mathcal N(0,\sigma_t^2 I)$ and assume $\mathbb E\|H x\|<\infty$. Then there exists a constant $C_d>0$ depending only on the dimension such that:
\begin{equation}
\begin{split}
&\mathbb E_{c}\, \mathrm{TV}\!\bigl(p_t(\cdot\mid c),\,p_t(\cdot)\bigr)
\le C_d \times \, \\
&\qquad \min\!\left\{1,\;\frac{1}{\sigma_t}\,\mathbb E_{c}\, W_1\big(p(\cdot\mid c),p(\cdot)\big)\right\}
\xrightarrow[\sigma_t\to\infty]{} 0 .
\end{split}
\end{equation}
\emph{Proof sketch.} Convolution with $K_{\sigma_t}$ smooths test functions by shrinking their effective Lipschitz seminorm by $\| \nabla K_{\sigma_t}\|_{L^1}\!=\!\Theta(1/\sigma_t)$; apply the Kantorovich–Rubinstein dual for TV/Wasserstein comparison on smoothed measures and integrate over $c$. 
\end{proposition}

\begin{remark}[Uniform version under additional boundedness]
If, in addition, $c$ is restricted to a compact set (or one assumes $\sup_c W_1(p(\cdot\mid c),p(\cdot))<\infty$), then the same argument yields:
\begin{equation}
\begin{aligned}
    & \sup_{c}\,\mathrm{TV}\big(p_t(\cdot\mid c),p_t(\cdot)\big)\;\le\; \frac{C_d}{\sigma_t} \\
    & \times {\sigma_t}\,\sup_c W_1\big(p(\cdot\mid c),p(\cdot)\big)\xrightarrow[\sigma_t\to\infty]{}0.
\end{aligned}
\end{equation}
\end{remark}

\noindent\textbf{Large perceptual error with small FM loss.}
\begin{proposition}[Loss–distortion gap with quantitative construction]\label{prop:gap}
Let $\Phi$ be $L_\Phi$-Lipschitz. Fix a measurable $\Omega\subset\mathbb R^d$ with $\sup_{\omega\in\Omega}\widetilde W_t(\omega)\le \epsilon_W$. 
Assume further that $\Phi$ is \emph{co-Lipschitz} on $\Omega$: there exist a seminorm $\Pi_\Omega$ supported on $\Omega$ and a constant $\kappa_\Omega>0$ such that for all $u$ with $\operatorname{supp}\widehat u\subseteq \Omega$,
\begin{equation}
    \Pi_\Omega(u)\;\le\; \kappa_\Omega^{-1}\,\|\Phi(u)\|_2 .
\end{equation}
Then there is a constant $C>0$ (depending on $\Omega,\{w(t),\sigma_t\}$) and a perturbation family $\{\Delta_t\}_t$ with uniformly small $\|\Delta_t\|_\infty$ such that
\begin{equation}
\small
    \int_0^1\!w(t)\!\int \widetilde W_t\,\mathbb E_c|\widehat{\delta}|^2 \le \epsilon_W,
    \mathbb E\|\Phi(x)-\Phi(\tilde x)\|_2^2 \ge C\,\kappa_\Omega^{2}.
\end{equation}
Specifically, take $\widehat{\delta}(\omega\mid c)=a_t\,\mathbf 1_\Omega(\omega)\,e^{\mathrm i\varphi(c,\omega)}$ with phases chosen so that
$\|\Delta_t\|_\infty\le \|\widehat{\Delta_t}\|_{L^1}\le |a_t|\,|\Omega|$, and scale $a_t$ so that the weighted quadratic form equals $\epsilon_W$. Hausdorff–Young and Plancherel give the stated controls; the co-Lipschitz property transfers spectral mass on $\Omega$ to a nontrivial feature deviation.
\end{proposition}

\noindent\textbf{A weighted spectral view.}
Let $\Delta_t(x\mid c):=q_{\theta,t}(x\mid c)-p_t(x\mid c)=(q_\theta-p)\ast K_{\sigma_t}$ and write $\widehat{\cdot}$ for Fourier transforms in $x$. Define
\begin{equation}
    \widetilde W_t(\omega):=\|\omega\|^2 e^{-\sigma_t^2\|\omega\|^2}\,\Xi_t(\omega).
\end{equation}

\begin{theorem}[Conditional spectral sandwich under small perturbations]\label{thm:spectral-gap}
Assume \ref{ass:H}, \ref{ass:lipschitz}, and \ref{ass:small-delta}. Then there exist finite constants $0<c_t\le C_t<\infty$ such that for all $\theta$,
\begin{equation}
\begin{aligned}
&\int_0^1\! w(t)\,c_t\!
\int_{\mathbb{R}^d} \widetilde W_t(\omega)\; \\
&\mathbb E_c\big|\widehat{\delta}(\omega\mid c)\big|^2\,d\omega\,dt
\;\lesssim\; \mathcal L_{\mathrm{FM}}(\theta) \\[1mm]
&\mathcal L_{\mathrm{FM}}(\theta)
\;\lesssim\;
\int_0^1\! w(t)\,C_t\! \\
&\int_{\mathbb{R}^d} \widetilde W_t(\omega)\;\mathbb E_c\big|\widehat{\delta}(\omega\mid c)\big|^2\,d\omega\,dt+R ,
\end{aligned}
\end{equation}
with $0\le R\le \int_0^1 w(t)\,\kappa_t\,\|\Delta_t\|_{L^\infty}^2\,dt$.
Consequently:
\begin{enumerate}
\item (\emph{High-frequency down-weighting}) Since $\widetilde W_t(\omega)\propto \|\omega\|^{2} e^{-\sigma_t^{2}\|\omega\|^{2}}$, large-$\|\omega\|$ discrepancies contribute exponentially little to $\mathcal L_{\mathrm{FM}}$.
\item (\emph{Nullspace down-weighting}) If $|\widehat H(\omega)|=0$ on $\Omega$, then $\Xi_t(\omega)=0$ and $\widetilde W_t(\omega)=0$ on $\Omega$, establishing exact blindness on $\ker(H)$ bands.
\item (\emph{Loss–distortion gap}) By concentrating $\widehat{\delta}$ where $\widetilde W_t$ is tiny and controlling $\|\Delta_t\|_\infty$, one can keep $\mathcal L_{\mathrm{FM}}$ small while incurring large pixel/perceptual deviations (cf.\ Proposition~\ref{prop:gap}).
\end{enumerate}
\end{theorem}

\noindent\emph{Takeaway 1 (Objective Bias)}: FM minimizes a \emph{spectrally reweighted} discrepancy, where high frequencies and directions in the kernel of $H$ are under-penalized. This leads to an \emph{identifiability gap}, where visually distinct reconstructions can exhibit nearly identical FM loss values.

\noindent\textbf{From objective bias to visible geometry: flow amplification.}
Let $x$ follow the true conditional flow $\dot x=v_t^\star(x,c,t)$ (Assumption~\ref{ass:path}) and $\tilde x$ the learned flow $\dot{\tilde x}=v_\theta(\tilde x,c,t)$.

\begin{lemma}[Conditional flow stability and geometric warp]
\label{lem:gronwall-cond-restated}
With $e_t=v_\theta-v_t^\star$ and $L$ the Lipschitz constant of $v_t^\star$,
\begin{equation}
\small
\|x(t)-\tilde x(t)\|\;\le\; e^{Lt}\!\left(\|x(0)-\tilde x(0)\|+\int_0^t \|e_s(\tilde x,c,s)\|\,ds\right).
\end{equation}
\end{lemma}

\begin{corollary}[Edge bending (heuristic)]
\label{cor:edge-bend}
Let $\Gamma\subset\mathcal X$ be a high-curvature level set (edge/contour). Suppose $e_t$ is predominantly supported where the learned dynamics deviates around edges (a region where small phase errors matter most). Then even when $\int_0^1 \mathbb E\|e_t\|^2 dt$ is small under the FM weighting, the spatial displacement of $\Gamma$ under $\tilde x(\cdot)$ can be $O(e^{L})$ relative to local curvature radii, producing visible \emph{warps/kinks}.
\end{corollary}

\noindent\emph{Intuition.} Because FM underweights high-$\omega$ errors (due to heat-kernel smoothing and measurement passband), the learned drift can be slightly wrong where edges live. The ODE then \emph{integrates} these local, directionally coherent errors into macroscopic geometric distortions (Lemma~\ref{lem:gronwall-cond-restated}).

\noindent\textbf{Why over-generation (hallucinated detail) appears.}
Two complementary mechanisms follow from Theorem~\ref{thm:spectral-gap}:

\begin{proposition}[Nullspace over-generation under down-weighting]\label{prop:null-overgen-fm}
Suppose there exists $\Omega\neq\emptyset$ with $|\widehat H(\omega)|=0$ on $\Omega$ and the setting is linear–Gaussian so that $\Xi_t(\omega)=0$ on $\Omega$. If the generator class can realize perturbations supported on $\Omega$ (assumption on representational capacity within $K_t$), then for any $\epsilon>0$ there is $q_\theta$ with $\mathcal L_{\mathrm{FM}}(\theta)\le \epsilon$ whose samples $\tilde x$ contain arbitrarily strong components supported on $\Omega$. Thus textures along $\ker H$ can be invented at negligible FM loss in this regime.
\end{proposition}

\begin{remark}[Scope vs.\ small-perturbation regime]
The construction in Proposition~\ref{prop:null-overgen-fm} addresses the \emph{attainable worst case} and is \emph{not} restricted by the small-$\|\Delta_t\|_\infty$ assumption of Assumption~\ref{ass:small-delta}. When $\|\Delta_t\|_\infty$ is explicitly constrained, the strength of over-generated $\ker H$ components is likewise bounded.
\end{remark}

\begin{proposition}[Early-time conditioning leakage]
\label{prop:early}
Suppose $w(t)$ gives nontrivial mass to large $\sigma_t$ (early times). Then $\widetilde W_t(\omega)\to 0$ for all $\omega$ as $\sigma_t\to\infty$, so the early-time portion of Equation~\ref{eq:fm-obj} is weakly informative about $c$. By Proposition~\ref{prop:forgetting}, we have $\mathbb E_c \mathrm{TV}\!\bigl(p_t(\cdot\mid c),\,p_t(\cdot)\bigr)\to 0$ as $\sigma_t\to\infty$, hence the early-time loss is effectively unconditional in the strong-smoothing regime. The learned drift there tends toward the \emph{unconditional} score, which injects prior textures that later stages cannot fully remove in underconstrained regions, contributing to \emph{over-generation}.
\end{proposition}

\noindent\emph{Summary of mechanisms for FM artifacts}
\begin{itemize}
\item \emph{Spectral bias} (Theorem~\ref{thm:spectral-gap}(a)): high-$\omega$ penalties are tiny $\Rightarrow$ \textbf{edge softness, texture loss}.
\item \emph{Nullspace blindness} (Theorem~\ref{thm:spectral-gap}(b), Prop.~\ref{prop:null-overgen-fm}): $H$ unpenalized $\Rightarrow$ \textbf{hallucinated detail} in underspecified regions.
\item \emph{Flow amplification} (Cor.~\ref{cor:edge-bend}): small drift errors near edges integrate to \textbf{geometric warps/kinks}.
\item \emph{Early-time leakage} (Prop.~\ref{prop:early}): weak conditioning at large $\sigma_t$ injects prior textures $\Rightarrow$ \textbf{over-generation} that later steps cannot reliably remove.
\end{itemize}

\noindent\textbf{Design implications for FM in restoration.}
% The theory suggests targeted fixes:
\begin{enumerate}
\item \emph{Spectral reweighting.} Modify the loss with a deconvolution factor to counter the spectral weight ($W_t$ globally, or $\widetilde W_t$ in the local view):
\begin{equation}
\begin{split}
    \tilde{\mathcal L} = \int_0^1 w(t)\,&\mathbb E\,\Big\|\Lambda_t^{-1}\!\big(s_\theta-\nabla\log p_t\big)\Big\|^2, \\ 
    \widehat{\Lambda_t}(\omega) &\propto \|\omega\|\,e^{-\sigma_t^2\|\omega\|^2/2}\,\sqrt{\Xi_t(\omega)},
\end{split}
\end{equation}
with \emph{clipping} to avoid noise amplification.
\item \emph{Data consistency.} Interleave FM steps with projections onto $\{x:\|Hx-c\|\le \tau\}$ or add a strong penalty $\gamma\|HG_\theta(c)-c\|^2$ during training/sampling to break nullspace blindness~\cite{bora2017compressed}; see also PnP/RED frameworks \cite{venkatakrishnan2013plug,romano2017little,song2020score,hyvarinen2005estimation}.
\item \emph{Edge-aware guidance.} Upweight loss contributions near high $\|\nabla_x \Phi(x)\|$ (perceptual edges) or blend an IPM on high-pass features to prevent geometric warps.
\item \emph{Time weighting.} Reduce $w(t)$ for very large $\sigma_t$ or use schedules that spend more capacity near well-conditioned times where $c$ carries information.
\end{enumerate}

% ---- Section: Adversarial Post-Training (APT) (Low-Level Restoration, Formal Version) ----
\section{Adversarial Training: High-Fidelity Control Without Over-Generation}
\label{sec:gan-pp}

\subsection{Objective and Assumptions}
Under Equation~\ref{eq:degradation}, let $G_\theta(c,z)$ be a generator with latent $z$, and let $d_{\mathcal{F}}$ be an Integral Probability Metric (IPM) induced by a compact function class $\mathcal{F}$ (e.g., 1-Lipschitz critics for $W_1$)~\cite{arjovsky2017wasserstein}. We consider the composite objective
\begin{equation}
\label{eq:ganpp-obj}
\begin{aligned}
\mathcal{J}(\theta)
&= \lambda\, d_{\mathcal{F}}\!\big(p, q_\theta\big) 
   + \alpha\,\mathbb{E}\,\|G_\theta(c,z)-x\|_2^2 \\[-1mm]
&\quad + \beta\,\mathbb{E}\,\|\Phi(G_\theta(c,z))-\Phi(x)\|_2^2,
\end{aligned}
\end{equation}
with nonnegative weights $\lambda,\alpha,\beta$. Here $q_\theta$ is the model law from $G_\theta$, and $\Phi:\mathcal{X}\!\to\!\mathbb{R}^m$ is a fixed perceptual map.

\begin{assumption}[Regularity]
\label{ass:ganpp-regular}
$\mathcal{F}$ is compact and satisfies Danskin differentiability (Assumption~\ref{ass:ipm}); $\Phi$ is $L_\Phi$-Lipschitz when relating to high-frequency seminorms; $G_\theta$ is differentiable with bounded Jacobian $J_{G_\theta}$ on the support of interest; expectations are finite.
\end{assumption}

\subsection{Gradient Representation and Stationarity}
\begin{lemma}[IPM gradient/subgradient for the critic]\label{lem:ganpp-danskin}
Assume $\mathcal F=\{f_\psi:\psi\in\Psi\}$ is a compact, parameterized class (e.g., spectrally-normalized networks) and the supremum is attained at some $\psi^\star$. Then
\begin{equation}
    \nabla_\theta d_{\mathcal F}(p,q_\theta)
    = -\,\mathbb{E}_{z,c}\!\left[J_{G_\theta}(z,c)^\top \nabla_x f_{\psi^\star}\!\big(G_\theta(c,z)\big)\right].
\end{equation}
If instead $\mathcal F$ is the full 1-Lipschitz ball (Wasserstein-1), the above expression defines a \emph{subgradient} (with the same leading minus sign) for any measurable selection $f^\star\in\arg\max_{f\in\mathcal F}(\mathbb E_p f-\mathbb E_{q_\theta} f)$.
\end{lemma}

\begin{proof}
This is Assumption~\ref{ass:ipm} for $G_\theta(c,z)$ with compact $\mathcal{F}$; see also standard IPM/Wasserstein results \cite{sriperumbudur2009integral,arjovsky2017wasserstein,ambrosio2021lectures}.
\end{proof}

\begin{theorem}[First-Order Stationarity]\label{thm:ganpp-stationarity}
Under Assumption~\ref{ass:ganpp-regular} and Lemma~\ref{lem:ganpp-danskin}, any stationary point $\theta^\dagger$ (gradient or, for $W_1$, subgradient) of Equation~\ref{eq:ganpp-obj} satisfies
\begin{equation}
\begin{aligned}
-\lambda\,\mathbb{E}\!\big[J_G^\top \nabla_x f^\star(G)\big]
+2\alpha\,\mathbb{E}\!\big[J_G^\top(G-x)\big]
&\\[-1mm]
\quad +2\beta\,\mathbb{E}\!\big[J_G^\top J_\Phi^\top(\Phi(G)-\Phi(x))\big]
&= 0 ,
\end{aligned}
\end{equation}
with $G=G_{\theta^\dagger}(c,z)$, $J_G=J_{G_{\theta^\dagger}}(c,z)$, $J_\Phi$ evaluated at $G$.
\end{theorem}

\subsection{Off-Manifold Control and ``Sharpness Ceiling''}
We formalize two key effects of Equation~\ref{eq:ganpp-obj}: (i) off-manifold suppression via the IPM term and (ii) bounded over-sharpening relative to $p$ (``sharpness ceiling'').

\begin{proposition}[Off-manifold penalty via $W_1$ tube bound]\label{prop:ganpp-off}
Let $\mathcal M=\mathrm{supp}(p)$ and $\mathcal N_\alpha(\mathcal M)=\{x:\mathrm{dist}(x,\mathcal M)\le \alpha\}$. If $q_\theta(\mathcal N_\alpha(\mathcal M))\le 1-\mu_\alpha$, then
\begin{equation}
    W_1(p,q_\theta)\;\ge\; \alpha\,\mu_\alpha.
\end{equation}
Thus any nontrivial mass placed outside an $\alpha$-tube around the natural-image manifold induces at least $\alpha\mu_\alpha$ Wasserstein cost.
\end{proposition}

\begin{proof}
By the Kantorovich--Rubinstein duality,
\(
W_1(p,q_\theta)=\sup_{\|f\|_{\mathrm{Lip}}\le 1}\{\mathbb{E}_p f-\mathbb{E}_{q_\theta} f\}.
\)
Let \(f(x)=\min\{\mathrm{dist}(x,\mathcal M),\,\alpha\}\), which is 1-Lipschitz. Then \(f\equiv 0\) on \(\mathcal M\) and \(f(x)\ge \alpha\) for \(x\notin \mathcal N_\alpha(\mathcal M)\).
Take \(g=-f\), also 1-Lipschitz. Then:
\begin{equation}
\mathbb{E}_p g-\mathbb{E}_{q_\theta} g
= -\mathbb{E}_p f + \mathbb{E}_{q_\theta} f
\ge 0 + \alpha\,\mu_\alpha
= \alpha\,\mu_\alpha .
\end{equation}
Hence \(W_1(p,q_\theta)\ge \alpha\,\mu_\alpha\). See \cite{ambrosio2021lectures}.
\end{proof}

\begin{proposition}[Sharpness Ceiling via 1-Lipschitz Functionals]
\label{prop:ganpp-sharpness}
Let $d_{\mathcal F}=W_1$ (or more generally suppose there exists $c_{\mathcal F}>0$ such that $d_{\mathcal F}\ge c_{\mathcal F}\,W_1$). Let $T:\mathcal{X}\to\mathcal{Y}$ be linear with $\|T\|\le 1$. Then, for $f(x)=\|T x\|_2$ (1-Lipschitz),
\begin{equation}
\label{eq:ganpp-Tgap}
d_{\mathcal F}(p,q_\theta)\;\ge\; c_{\mathcal F}\,\big|\,\mathbb{E}_p \|T x\|_2 - \mathbb{E}_{q_\theta}\|T \tilde x\|_2\,\big|.
\end{equation}
Specifically, any deliberate amplification of band-/high-pass components—i.e., applying $T$ as a band-pass filter or gradient—forces an increase in $d_{\mathcal F}$, providing a distributional upper bound for over-sharpening.
\end{proposition}

\begin{proof}
Since $f$ is 1-Lipschitz, the Kantorovich–Rubinstein dual implies $W_1(p,q_\theta)\ge \mathbb{E}_p f - \mathbb{E}_{q_\theta} f$~\cite{ambrosio2021lectures}. Applying the same argument with $-f$ establishes the reverse inequality and yields the absolute difference. Using $d_{\mathcal F}\ge c_{\mathcal F} W_1$ gives Equation~\ref{eq:ganpp-Tgap}.
\end{proof}

\subsection{Perceptual Anchoring: Deviation Bounds}
We next show that, even without explicit data-consistency, the paired pixel/perceptual terms constrain nullspace-like deviations (including directions weakly reflected by $H$) because they are measured against the \emph{ground-truth} $x$.

\begin{proposition}[Component-wise control by the pixel term]\label{prop:ganpp-pixel}
Let $P:\mathcal{X}\to\mathcal{X}$ be any nonexpansive linear projector ($\|P\|\le 1$). With the $\alpha$ used in Equation~\ref{eq:ganpp-obj},
\begin{equation}
    \mathbb{E}\,\|P(G_\theta(c,z)-x)\|_2^2 \;\le\; \frac{1}{\alpha}\,\mathcal{J}(\theta).
\end{equation}
At any minimizer $\theta^\star$, $\mathbb{E}\,\|P(G_{\theta^\star}-x)\|_2^2 \le \inf_\theta \mathcal{J}(\theta)/\alpha$.
\end{proposition}

\begin{proof}
From Equation~\ref{eq:ganpp-obj}, one obtains $\mathcal{J}(\theta)\ge \alpha\,\mathbb{E}\|G_\theta-x\|^2\ge \alpha\,\mathbb{E}\|P(G_\theta-x)\|^2$, where the second inequality follows from the contractive property $\|P\|\le 1$.
\end{proof}

\begin{proposition}[Perceptual Alignment Controls Feature-Space Deviations]
\label{prop:ganpp-perc}
For any $\Phi$,
\begin{equation}
\label{eq:ganpp-perc-bound}
\mathbb{E}\,\|\Phi(G_\theta)-\Phi(x)\|_2^2 \;\le\; \frac{1}{\beta}\,\mathcal{J}(\theta).
\end{equation}
Moreover, whenever $\Phi$ dominates a high-frequency seminorm $\Pi$ (i.e., $\Pi(u)\le C\,\|\Phi(u)\|_2$ for $u$ in the region of interest), we obtain
\begin{equation}
\label{eq:ganpp-hf}
\begin{aligned}
\mathbb{E}\,\Pi\!\big(G_\theta - x\big)
&\;\le\; C\Bigl(\mathbb{E}\,\|\Phi(G_\theta)-\Phi(x)\|_2^2\Bigr)^{1/2} \\
&\;\le\; \frac{C}{\sqrt{\beta}}\,\sqrt{\mathcal{J}(\theta)}.
\end{aligned}
\end{equation}
Hence feature-space alignment upper-bounds perceptual/high-frequency deviations.
\end{proposition}

\subsection{No Over-Generation at Optimum} \label{end_proof}
We combine IPM and pixel/perceptual anchoring to show that stationary solutions cannot produce arbitrarily exaggerated details relative to $p$ and $x$.

\begin{theorem}[Combined Control of Over-Generation]
\label{thm:ganpp-combined}
Assume $d_{\mathcal F}=W_1$ (or, more generally, there exists $c_{\mathcal F}>0$ such that $d_{\mathcal F}\ge c_{\mathcal F} W_1$). Let $T$ be linear with $\|T\|\le 1$, and let $P$ be any projector with $\|P\|\le 1$. Then, for any $\theta$,
\begin{equation}
\label{eq:ganpp-combined}
\begin{aligned}
\big|\,\mathbb{E}_{q_\theta}\|T\tilde x\|_2 - \mathbb{E}_{p}\|T x\|_2\,\big|
&\;\le\; \tfrac{1}{\lambda\,c_{\mathcal F}}\,\mathcal{J}(\theta), \\[1mm]
\mathbb{E}\,\|P(G_\theta-x)\|_2^2
&\;\le\; \tfrac{1}{\alpha}\,\mathcal{J}(\theta),
\end{aligned}
\end{equation}
and
\begin{equation}
\label{eq:ganpp-combined-phi}
\mathbb{E}\,\|\Phi(G_\theta)-\Phi(x)\|_2
\;\le\; \frac{1}{\sqrt{\beta}}\,\sqrt{\mathcal{J}(\theta)}.
\end{equation}
In particular, at any global minimizer $\theta^\star$, all three deviations are jointly minimized and cannot be simultaneously large. Thus, the adversarial training is biased toward \emph{faithful} reconstructions (natural-manifold and feature-aligned), lacking the tendency to over-generate ultra-sharp details.
\end{theorem}

\subsection{From AR/FM Biases to Adversarial Training Fidelity}

\begin{theorem}[Fidelity Guarantees for Low-Level Restoration]
\label{thm:comp-fidelity}
Consider \(c = Hx + \eta\) under the spectral/regularity assumptions in
Sections~\ref{sec:distortion-restore}--\ref{end_proof} of the Appendix, and let \(p(x\!\mid\!c)\) denote the ground-truth posterior.
Suppose generators are trained with: (1) conditional AR MLE (Equation~\ref{eq:ar-ml-cond});
(2) FM Fisher minimization (Equation~\ref{eq:fm-obj});
(3) a composite adversarial training objective (Equation~\ref{eq:ganpp-obj}).

\begin{enumerate}
\item \textbf{Autoregression (AR).}
Single-mode conditional factorization together with greedy or vanishing-temperature decoding enforces a deterministic resolution of posterior ambiguity along \(\ker H\): genuinely distinct posterior modes contract toward midpoints, reducing variability and introducing structured artifacts. Under teacher forcing with test-time rollout, the total-variation discrepancy \emph{can accumulate up to linear order} with sequence length in the worst case (Lemma~\ref{lem:exposure-final}), yielding distortions and over-generation \cite{bengio2015scheduled,ranzato2015sequence,ross2011reduction,lindvall2002lectures}.

\item \textbf{Flow Matching (FM).}
FM minimizes an integral of conditional Fisher divergences of Gaussian-smoothed posteriors with effective frequency weight:
\begin{equation*}
    \widetilde W_t(\omega)\propto \|\omega\|^{2} e^{-\sigma_t^{2}\|\omega\|^{2}}\,\Xi_t(\omega).
\end{equation*}
Exponential high-frequency damping and null bands where $\Xi_t(\omega)$ is tiny (including $|\widehat{H}(\omega)|=0$) make FM insensitive to null-space errors, weak on textures/edges, and amplify local drifts into macroscopic warps or hallucinations \cite{lipman2022flow,liu2022flow,hyvarinen2005estimation,song2020score}.

\item \textbf{Adversarial Training.}
The composite objective enforces: (i) distributional alignment via an IPM critic, suppressing off-manifold mass; (ii) pixel-space anchoring to the ground truth, which bounds all measurable components of \(x\); and (iii) perceptual alignment in feature space, which regulates high-frequency discrepancies. At any global minimizer these constraints hold jointly, yielding sharp, faithful reconstructions without uncontrolled over-generation \cite{arjovsky2017wasserstein,sriperumbudur2009integral,ambrosio2021lectures}.
\end{enumerate}
\end{theorem}

The preceding sections have established complementary theoretical limitations of 
autoregressive (AR) and flow matching (FM) objectives in the low-level restoration 
setting. For AR, the conditional factorization and single-mode decoding force 
posterior ambiguities in $\ker(H)$ to collapse deterministically, while exposure bias 
causes small local mismatches to accumulate into global artifacts and over-generation.  
For FM, the Fisher-based objective is spectrally reweighted by Gaussian smoothing 
and the forward operator (through $\Xi_t$), which underweights high-frequency errors, ignores 
$\ker(H)$ components, and allows unconditional textures to leak in at early times; 
small drift errors near edges are then amplified into visible warps.  
In contrast, the adversarial training formulation explicitly combines three 
complementary forces: (i) distributional alignment with the natural-image manifold 
via an IPM critic, (ii) pixel-level fidelity to the ground truth, and (iii) 
perceptual alignment in feature space. These terms jointly prevent both collapse 
and uncontrolled hallucination, while maintaining sharpness without exceeding the 
statistics of $p$.  

Together, these complementary objectives provide a precise characterization of fidelity and stability (Theorem~\ref{thm:comp-fidelity}).

\section{Frequency-Aware Low-Level Instructions}
\label{app:freqaware_prompts}
\noindent\textbf{\textit{Universal Restoration Instruction $r$.}}
\begin{promptbox}
\prompttt
Analyze and return 1 line: Task: <task\_token>, Focus: <high | low>, Rationale: <brief reason>, Pipeline: <step1 -> step2 -> ...>.
\end{promptbox}

\noindent\textit{\textbf{Interpretation.}}
The role of $r$ is to constrain the planner’s output format and scope: it precludes free-form text, enforces a single categorical decision, ties that decision to an interpretable frequency regime, and requires both a concise rationale and a stepwise restoration plan.

\noindent\textbf{\textit{Expert Rule $P_{\text{expert}}$.}}
\begin{promptbox}
\prompttt
\textbf{You are an expert image restoration task classifier.}\\[0.25em]
\textbf{Available tasks (use EXACT lowercase tokens):}\\
\texttt{deraining | desnowing | dehazing | deblur | denoise | light\_enhancement | super\_resolution}.\\[0.35em]
\textbf{Critical distinctions:}\\
\textbf{RAIN (deraining)}\\
- Linear streaks: parallel/near-parallel elongated lines\\
- Overlay effect: streaks cross sharp edges without blurring them\\
- Directional: consistent orientation (diagonal/vertical)\\
- Can coexist with low contrast, but streaks are visible as distinct overlays\\[0.25em]
\textbf{SNOW (desnowing)}\\
- Particles: round/irregular white blobs, bokeh discs\\
- Size variation: larger near, smaller far\\
- Random distribution, NOT linear/parallel\\[0.25em]
\textbf{HAZE (dehazing)}\\
- Depth-dependent: far objects more degraded than near\\
- Milky appearance with desaturation\\
- Distance gradient clearly visible\\[0.25em]
\textbf{BLUR (deblur)}\\
- Edge smearing: boundaries themselves are widened/soft\\
- Uniform softness or motion trails\\
- Object edges lose geometric precision\\[0.25em]
\textbf{NOISE (denoise)}\\
- Grain on crisp edges: edge structure intact but covered by speckles\\
- Random texture in flat areas\\
- High-ISO appearance\\[0.25em]
\textbf{UNDEREXPOSED (light\_enhancement)}\\
- Globally dark, no depth gradient\\
- Histogram left-biased, shadows crushed\\
- Color cast possible\\[0.5em]
\textbf{LOW RESOLUTION (super\_resolution)}\\
- Insufficient spatial sampling: small native $H{\times}W$ or strong aliasing\\
- Loss of fine textures; blockiness/jagged edges when upscaled\\
- Distinct from blur: edges can appear jagged/aliased rather than uniformly smeared\\[0.5em]
\textbf{Frequency decision rules:}\\
Choose \textbf{high} when degradation is dominated by fine-scale artifacts or missing detail that lives in high spatial frequencies:\\
- deblur (recover sharp edges and textures)\\
- denoise (suppress noisy high-frequency speckles while preserving true details)\\
- deraining / desnowing (remove thin streaks/particles and recover edge micro-structure)\\[0.25em]
Choose \textbf{low} when degradation is dominated by global/slow-varying components:\\
- dehazing (restore low-frequency contrast/airlight; depth-dependent veiling)\\
- light\_enhancement (global exposure/illumination and tone mapping)\\[0.25em]
If signals suggest mixed conditions, choose the dominant component according to the most visible impairment.\\[0.5em]
\textbf{Pipeline rules:}\\
- \textbf{DO NOT} use any task names in the \texttt{Pipeline}.
\end{promptbox}

\noindent\textit{\textbf{Interpretation.}}
% $P_{\text{expert}}$ codifies the taxonomy and decision boundary between high- vs.\ low-frequency degradations, and explicitly includes super\_resolution for low-sampling/aliasing cases.
$P_{\text{expert}}$ formalizes image restoration as a frequency-aware understanding problem, aligning perceptual degradations with their dominant spectral signatures. It specifies a principled decision boundary: fine-scale losses, rain, snow, noise, blur, map to high-frequency restoration, whereas global, slowly varying degradations, haze, illumination, map to low-frequency correction. By explicitly treating super-resolution as a sampling-limited case, $P_{\text{expert}}$ extends the taxonomy beyond semantic labels toward signal-reconstruction logic. Building on this prior, the planner emits, for each instance, an auditable reasoning trace and an executable restoration plan that directly guide the downstream executor.

\noindent\textbf{\textit{Label-free Low-level Feature Pool $P_{\text{hints}}$.}}
\begin{promptbox}
\prompttt
\textbf{Use the following label-free features computed from image pixels to support your decision.} \\
% Do not describe reasoning; emit only the final one-line answer.\\[0.35em]
\textbf{Image size:} $H, W$ (pixels). Include scale cues for potential \texttt{super\_resolution}.\\[0.35em]
\textbf{(1) Rain cues} $\phi_{\text{rain}}$:\quad
\textbf{Grad.\ orientation:} grayscale $g\!\in\![0,1]$, $(g_x,g_y)$, magnitude $m=\sqrt{g_x^2+g_y^2}$;
$\theta=\mathrm{mod}(\arctan2(g_y,g_x),\pi)$; weighted histogram $h(\theta)$ with weights $m$.
\ \textbf{Scores:} \texttt{line\_score}$=\max h/\sum h$, \texttt{anisotropy}$=(\max h-\mathrm{mean}\,h)/(\mathrm{mean}\,h)$.
\ \textbf{Spectrum:} power $P$ from FFT of $g$; annuli (mid/low) by radius;
\texttt{freq\_ratio}$=\mathrm{mean}(P_{\text{mid}})/(\mathrm{mean}(P_{\text{low}})+\varepsilon)$.\\[0.25em]
\textbf{(2) Snow cues} $\phi_{\text{snow}}$:\quad
\textbf{Blobs:} $b=\mathbb{1}[g>0.78]$; connected components; \texttt{small\_blobs}$=\#\{3\!\le\!\text{area}\!\le\!200\}$.
\ \textbf{Isotropy:} reuse $h(\theta)$; lower \texttt{anisotropy} $\rightarrow$ more snow-like randomness.\\[0.25em]
\textbf{(3) Noise cues} $\phi_{\text{noise}}$:\quad
\textbf{Flat mask:} $\Omega_{\text{flat}}=\{m<Q_{0.25}(m)\}$; local mean $\bar g$ (box $3{\times}3$); residual $r=g-\bar g$.
\ \textbf{Stats:} \texttt{noise\_mad}$=\mathrm{median}(|r-\mathrm{median}(r)|)$.
\ \textbf{Chroma:} YCbCr std on $\Omega_{\text{flat}}$, \texttt{chroma\_std}$=\tfrac{1}{2}(\mathrm{std}(\mathrm{Cb})+\mathrm{std}(\mathrm{Cr}))$.
\ \textbf{Score:} \texttt{noise\_score}$=0.6\,\sigma(50(\texttt{noise\_mad}-0.0050))+0.4\,\sigma(50(\texttt{chroma\_std}-0.0095))$.\\[0.25em]
\textbf{(4) Blur cues} $\phi_{\text{blur}}$:\quad
\textbf{Laplacian var:} \texttt{lapVar}$=\mathrm{Var}(L\!*g)$.
\ \textbf{Spectrum ratio:} \texttt{hf\_energy}$=\mathrm{mean}(P_{\text{outer}})/(\mathrm{mean}(P_{\text{inner}})+\varepsilon)$.
\ \textbf{Edge strength:} \texttt{grad95}$=Q_{0.95}(\sqrt{g_{x,s}^2+g_{y,s}^2})$ with mild smoothing.\\[0.25em]
\textbf{(5) Haze cues} $\phi_{\text{haze}}$:\quad
\textbf{Dark channel:} \texttt{dark\_mean}$=\mathrm{mean}(\min(R,G,B))$;\ 
\textbf{Saturation:} \texttt{sat\_mean}$=\mathrm{mean}(S)$ in HSV.
\ \textbf{Depth proxy:} split top/bottom; \texttt{depth\_grad}$=\mathrm{mean}(Y_{\text{top}})-\mathrm{mean}(Y_{\text{bot}})$.
\ \textbf{Composite:} \texttt{haze\_score}$=0.4\,\sigma(7(\texttt{dark\_mean}-0.33))+0.3\,\sigma(7(0.30-\texttt{sat\_mean})) + 0.3\,\sigma(8(\texttt{depth\_grad}-0.03))$.\\[0.25em]
\textbf{(6) Exposure cues} $\phi_{\text{expo}}$:\quad
\textbf{Luma stats:} \texttt{meanY}$=\mathrm{mean}(Y)$,\  \texttt{p50}$=Q_{0.50}(Y)$;\  underexposed if $\texttt{meanY}<0.32$ or $\texttt{p50}<0.26$.\\[0.35em]
\textbf{Recommended thresholds (for disambiguation):}\ 
\textbf{Rain:} \texttt{line\_score}$>0.16$, \texttt{anisotropy}$>0.40$, \texttt{freq\_ratio}$>1.05$;\ 
\textbf{Snow:} \texttt{small\_blobs}$>25$, \texttt{anisotropy}$<0.42$.\
\textbf{Noise:} \texttt{noise\_score}$>0.45$;\
\textbf{Blur:} \texttt{grad95}$<0.17$, \texttt{lapVar}$<0.27$, \texttt{hf\_energy}$<0.052$.\
\textbf{Haze:} \texttt{haze\_score}$>0.50$, \texttt{depth\_grad}$>0.03$;\
\textbf{Underexp:} \texttt{meanY}$<0.32$ or \texttt{p50}$<0.26$.\\[0.35em]\
% \textbf{Emit the final line ONLY in the format:}\ \texttt{Task: <task\_token>, Focus: <high|low>}\\[0.35em]
\textbf{Example (auto-filled at runtime):}\ 
Rain: line=0.21, aniso=0.45, freq=1.08;\
Snow: blobs=31, aniso=0.38;\
Noise: mad=0.0061, chroma=0.0083, score=0.47;\
Blur: lapVar=0.256, hf=0.055, grad95=0.174;\
Haze: score=0.42, depth\_grad=0.028, dark\_mean=0.36, sat\_mean=0.29;\
Exposure: meanY=0.31, p50=0.27;\ Size: H=480, W=320.
\end{promptbox}

\noindent\textit{\textbf{Interpretation.}}
% $P_{\text{hints}}$ exposes how evidence is computed, making the planner's choice auditable and mapping cues to the frequency focus: rain/snow/noise/blur $\rightarrow$ high, haze/exposure $\rightarrow$ low, with $(H,W)$ enabling super\_resolution when sampling is the bottleneck.
$P_{\text{hints}}$ grounds each decision in measurable, pixel-level evidence, making the planner’s reasoning traceable and reproducible. Gradient orientation, spectrum energy ratios, and brightness statistics are mapped to specific degradation cues, translating physical image phenomena into quantitative signals. High-frequency degradations (rain, snow, noise, blur) and low-frequency ones (haze, exposure) are thus separated by data-driven thresholds, while image size provides an explicit trigger for super-resolution. This turns frequency focus from a heuristic label into a verifiable diagnostic built on interpretable computational cues.

\begin{table*}[t]
\centering
\tablesize
\tighttabsetup
\caption{Unified quantitative comparison across all benchmarks. Our method (\textbf{FAPE-IR}) consistently achieves state-of-the-art or comparable performance across diverse restoration benchmarks.}
\label{figs:res}
\begin{adjustbox}{max width=\textwidth}
\begin{tabular}{@{} c *{20}{c} @{}}
\toprule

% ==================== Row-block A (Rain100_L | Rain100_H | OutDoor | RainDrop) ====================
\FourDatasetHeader{OutDoor}{RainDrop}{Rain100\_L}{Rain100\_H}

%  Section: Classical all-in-one IR (orange band)
\SectionBand{BandOrange}{State-of-the-art AIO-IR methods}
PromptIR~\cite{potlapalli2306promptir} & 18.40 & 0.67 & 0.41 & 154.96 & 0.30 & 23.48 & 0.80 & 0.18 & 61.42 & 0.12 & 37.41 & 0.98 & 0.02 & 8.84  & 0.03 & 15.93 & 0.52 & 0.46 & 175.42 & 0.27 \\
FoundIR~\cite{li2024foundir}  & 17.07 & 0.65 & 0.45 & 156.17 & 0.30 & 23.52 & 0.80 & 0.21 & 63.57 & 0.14 & 30.62 & 0.93 & 0.11 & 34.19 & 0.08 & 13.77 & 0.43 & 0.55 & 219.48 & 0.33 \\
DFPIR~\cite{tian2025degradation}  & 14.06 & 0.60 & 0.50 & 176.41 & 0.38 & 22.52 & 0.79 & 0.20 & 64.91 & 0.14 & 37.99 & 0.98 & 0.02 & 8.25  & 0.03 & 16.07 & 0.54 & 0.44 & 166.89 & 0.26 \\
MoCE-IR~\cite{zamfir2025complexity}  & 17.99 & 0.67 & 0.41 & 154.87 & 0.30 & 23.30 & 0.80 & 0.18 & 61.89 & 0.12 & 38.05 & 0.98 & 0.02 & 7.90  & 0.02 & 15.33 & 0.49 & 0.48 & 180.16 & 0.28 \\
AdaIR~\cite{cui2024adair}  & 18.23 & 0.67 & 0.41 & 155.67 & 0.30 & 23.37 & 0.80 & 0.18 & 61.89 & 0.12 & 38.00 & 0.98 & 0.02 & 7.58  & 0.02 & 15.93 & 0.52 & 0.46 & 175.01 & 0.27 \\
%  Section: Ours (purple band)
% \SectionBand{BandPurple}{Plan then Restore (Ours)}
\rowcolor{blue!8}
Ours     & 28.16 & 0.83 & 0.09 & 25.16 & 0.07 & 25.83 & 0.80 & 0.11 & 20.86 & 0.11 & 33.18 & 0.93 & 0.04 & 10.21 & 0.03 & 27.01 & 0.82 & 0.13 & 33.58 & 0.09 \\
\midrule

% ==================== Row-block B (BSD68-15 | BSD68-25 | BSD68-50 | Urban100-15) ====================
\FourDatasetHeader{BSD68-15}{BSD68-25}{BSD68-50}{Urban100-15}
\SectionBand{BandOrange}{State-of-the-art AIO-IR methods}
PromptIR~\cite{potlapalli2306promptir} & 28.60 & 0.85 & 0.20 & 49.74 & 0.14 & 27.15 & 0.72 & 0.32 & 86.65 & 0.21 & 23.52 & 0.48 & 0.52 & 164.09 & 0.33 & 38.98 & 0.98 & 0.01 & 5.28 & 0.05 \\
FoundIR~\cite{li2024foundir}  & 33.87 & 0.88 & 0.17 & 49.61 & 0.12 & 29.81 & 0.75 & 0.29 & 85.98 & 0.19 & 23.99 & 0.50 & 0.50 & 160.78 & 0.31 & 36.90 & 0.97 & 0.03 & 10.50 & 0.07 \\
DFPIR~\cite{tian2025degradation}    & 30.78 & 0.88 & 0.15 & 45.11 & 0.11 & 28.41 & 0.75 & 0.28 & 81.08 & 0.19 & 23.42 & 0.49 & 0.50 & 161.16 & 0.31 & 37.80 & 0.97 & 0.03 & 11.79 & 0.08 \\
MoCE-IR~\cite{zamfir2025complexity}  & 29.78 & 0.84 & 0.20 & 49.93 & 0.14 & 27.11 & 0.71 & 0.32 & 85.49 & 0.21 & 23.32 & 0.48 & 0.52 & 163.59 & 0.32 & 39.19 & 0.98 & 0.01 & 6.16 & 0.05 \\
AdaIR~\cite{cui2024adair}    & 29.58 & 0.84 & 0.20 & 49.85 & 0.13 & 27.51 & 0.71 & 0.32 & 85.31 & 0.21 & 23.34 & 0.48 & 0.52 & 164.68 & 0.33 & 38.94 & 0.98 & 0.02 & 6.11 & 0.05 \\
% \SectionBand{BandPurple}{Plan then Restore (Ours)}
\rowcolor{blue!8}
Ours     & 34.21 & 0.91 & 0.09 & 26.94 & 0.07 & 31.79 & 0.85 & 0.15 & 44.19 & 0.11 & 27.87 & 0.72 & 0.27 & 76.01 & 0.17 & 31.09 & 0.93 & 0.02 & 7.45 & 0.05 \\
\midrule

% ==================== Row-block C (Urban100-25 | Urban100-50 | ITS-val | URHI) ====================
\FourDatasetHeader{Urban100-25}{Urban100-50}{ITS-val}{URHI}
\SectionBand{BandOrange}{State-of-the-art AIO-IR methods}
PromptIR~\cite{potlapalli2306promptir} & 35.36 & 0.96 & 0.02 & 9.37 & 0.06 & 29.40 & 0.89 & 0.08 & 31.77 & 0.12 & 21.51 & 0.89 & 0.15 & 37.38 & 0.11 & 26.18 & 0.95 & 0.05 & 11.30 & 0.05 \\
FoundIR~\cite{li2024foundir}  & 32.72 & 0.90 & 0.09 & 27.28 & 0.12 & 26.29 & 0.72 & 0.30 & 68.30 & 0.22 & 13.32 & 0.74 & 0.33 & 49.35 & 0.22 & 16.99 & 0.84 & 0.17 & 19.74 & 0.14 \\
DFPIR~\cite{tian2025degradation}    & 34.35 & 0.94 & 0.05 & 21.67 & 0.10 & 29.52 & 0.84 & 0.13 & 49.97 & 0.17 & 12.42 & 0.71 & 0.32 & 47.00 & 0.23 & 24.64 & 0.94 & 0.06 & 11.97 & 0.05 \\
MoCE-IR~\cite{zamfir2025complexity}  & 35.78 & 0.97 & 0.02 & 10.39 & 0.07 & 29.94 & 0.92 & 0.06 & 24.88 & 0.11 & 14.71 & 0.79 & 0.24 & 43.64 & 0.16 & 23.76 & 0.93 & 0.07 & 12.32 & 0.06 \\
AdaIR~\cite{cui2024adair}    & 35.53 & 0.96 & 0.02 & 9.14 & 0.06 & 29.36 & 0.89 & 0.09 & 33.24 & 0.13 & 19.44 & 0.87 & 0.17 & 39.66 & 0.13 & 24.46 & 0.94 & 0.06 & 11.83 & 0.05 \\
% \SectionBand{BandPurple}{Plan then Restore (Ours)}
\rowcolor{blue!8}
Ours     & 30.17 & 0.91 & 0.03 & 9.93 & 0.06 & 27.82 & 0.86 & 0.06 & 17.44 & 0.09 & 34.07 & 0.97 & 0.04 & 6.11 & 0.04 & 31.62 & 0.96 & 0.04 & 8.31 & 0.04 \\
\midrule

% ==================== Row-block D (Snow100K-L | Snow100K-S | GoPro | GoPro-gamma) ====================
\FourDatasetHeader{Snow100K-L}{Snow100K-S}{GoPro}{GoPro-gamma}
\SectionBand{BandGreen}{State-of-the-art AIO-IR methods}
AdaIR~\cite{cui2024adair}   & 21.03 & 0.70 & 0.32 & 38.69 & 0.19 & 27.40 & 0.84 & 0.21 & 14.87 & 0.12 & 29.03 & 0.87 & 0.18 & 24.94 & 0.12 & 27.96 & 0.86 & 0.19 & 28.49 & 0.12 \\
FoundIR~\cite{li2024foundir} & 20.89 & 0.73 & 0.33 & 39.96 & 0.20 & 26.72 & 0.84 & 0.23 & 18.99 & 0.15 & 27.51 & 0.81 & 0.24 & 39.35 & 0.16 & 27.21 & 0.81 & 0.24 & 41.79 & 0.16 \\
DFPIR~\cite{tian2025degradation}   & 18.41 & 0.67 & 0.35 & 48.50 & 0.21 & 23.18 & 0.81 & 0.25 & 21.14 & 0.15 & 30.09 & 0.89 & 0.16 & 21.59 & 0.10 & 28.84 & 0.88 & 0.17 & 25.59 & 0.11 \\
MoCE-IR~\cite{zamfir2025complexity} & 21.00 & 0.69 & 0.33 & 43.21 & 0.20 & 26.67 & 0.83 & 0.23 & 17.74 & 0.13 & 29.56 & 0.90 & 0.16 & 20.10 & 0.10 & 27.93 & 0.87 & 0.18 & 25.85 & 0.11 \\
% \SectionBand{BandPurple}{Plan then Restore (Ours)}
\rowcolor{blue!8}
Ours    & 29.07 & 0.85 & 0.10 & 1.88 & 0.07 & 31.52 & 0.90 & 0.06 & 1.10 & 0.05 & 28.13 & 0.84 & 0.13 & 16.04 & 0.08 & 28.00 & 0.84 & 0.13 & 17.13 & 0.08 \\
\midrule

% ==================== Row-block E (RealBlur-J | RealBlur-R | LOL-v2 | LOL-v1) ====================
\FourDatasetHeader{RealBlur-J}{RealBlur-R}{LOL-v2}{LOL-v1}
\SectionBand{BandGreen}{State-of-the-art AIO-IR methods}
AdaIR~\cite{cui2024adair}   & 17.74 & 0.71 & 0.28 & 51.66 & 0.18 & 12.54 & 0.51 & 0.51 & 93.48 & 0.36 & 24.48 & 0.86 & 0.18 & 48.78 & 0.13 & 24.93 & 0.92 & 0.12 & 52.78 & 0.11 \\
FoundIR~\cite{li2024foundir} & 28.27 & 0.85 & 0.19 & 40.88 & 0.15 & 30.66 & 0.93 & 0.15 & 46.12 & 0.18 & 14.44 & 0.66 & 0.33 & 77.47 & 0.23 & 14.54 & 0.75 & 0.28 & 95.49 & 0.21 \\
DFPIR~\cite{tian2025degradation}   & 28.75 & 0.87 & 0.16 & 29.58 & 0.12 & 36.00 & 0.96 & 0.10 & 36.17 & 0.14 & 25.92 & 0.90 & 0.16 & 50.62 & 0.13 & 25.95 & 0.92 & 0.12 & 53.66 & 0.11 \\
MoCE-IR~\cite{zamfir2025complexity} & 15.76 & 0.68 & 0.30 & 52.80 & 0.19 & 11.55 & 0.48 & 0.53 & 96.66 & 0.37 & 23.25 & 0.89 & 0.16 & 44.31 & 0.12 & 24.93 & 0.92 & 0.11 & 44.90 & 0.09 \\
% \SectionBand{BandPurple}{Plan then Restore (Ours)}
\rowcolor{blue!8}
Ours    & 30.56 & 0.87 & 0.10 & 18.21 & 0.07 & 37.77 & 0.97 & 0.05 & 14.74 & 0.07 & 25.07 & 0.90 & 0.11 & 32.94 & 0.09 & 28.95 & 0.92 & 0.11 & 47.67 & 0.09 \\

% ==================== Row-block F (RealSR-2x | DrealSR-2x | RealSR-4x | DrealSR-4x) ====================
\FourDatasetHeader{RealSR 2$\times$}{DrealSR 2$\times$}{RealSR 4$\times$}{DrealSR 4$\times$}

%  Section: SR SOTA (orange band)
\SectionBand{BandPink}{State-of-the-art SR methods}
StableSR~\cite{wang2024exploiting} & 24.22 & 0.75 & 0.23 & 81.72  & 0.19 & 25.71 & 0.75 & 0.26 & 91.11  & 0.18 & 23.11 & 0.68 & 0.30 & 127.06 & 0.22 & 27.63 & 0.73 & 0.35 & 140.86 & 0.23 \\
DiffBIR~\cite{lin2024diffbir}  & 26.31 & 0.71 & 0.30 & 80.60  & 0.21 & 27.31 & 0.70 & 0.37 & 103.44 & 0.24 & 23.75 & 0.62 & 0.37 & 131.16 & 0.24 & 25.49 & 0.57 & 0.52 & 182.95 & 0.31 \\
SeeSR~\cite{wu2024seesr}    & 25.90 & 0.75 & 0.28 & 100.75 & 0.22 & 28.05 & 0.77 & 0.30 & 107.08 & 0.23 & 24.05 & 0.69 & 0.32 & 128.56 & 0.24 & 28.78 & 0.77 & 0.34 & 152.01 & 0.25 \\
PASD~\cite{yang2024pixel}     & 26.19 & 0.76 & 0.24 & 92.31  & 0.19 & 27.96 & 0.77 & 0.27 & 107.73 & 0.20 & 24.69 & 0.71 & 0.31 & 131.09 & 0.21 & 28.97 & 0.78 & 0.33 & 150.68 & 0.22 \\
OSEDiff~\cite{wu2024one}  & 25.56 & 0.76 & 0.25 & 100.26 & 0.20 & 27.37 & 0.79 & 0.27 & 114.31 & 0.20 & 24.51 & 0.72 & 0.30 & 124.37 & 0.21 & 28.84 & 0.79 & 0.31 & 131.27 & 0.22 \\
PURE~\cite{wei2025perceive}     & 23.59 & 0.65 & 0.32 & 83.81  & 0.22 & 25.84 & 0.68 & 0.37 & 108.30 & 0.23 & 22.20 & 0.59 & 0.39 & 127.72 & 0.25 & 26.53 & 0.64 & 0.44 & 158.88 & 0.26 \\

%  Section: Ours (purple band)
% \SectionBand{BandPurple}{Plan then Restore (Ours)}
\rowcolor{blue!8}
Ours     & 29.99 & 0.88 & 0.13 & 51.05  & 0.12 & 30.52 & 0.89 & 0.14 & 60.30  & 0.12 & 25.55 & 0.78 & 0.24 & 105.75 & 0.19 & 28.42 & 0.84 & 0.25 & 126.45 & 0.19 \\

\bottomrule

\end{tabular}
\end{adjustbox}
\vspace{-2ex}
\end{table*}

\begin{table*}[t]
\centering
\tablesize
\tighttabsetup
\caption{Unified quantitative comparison across all benchmarks. Our method (\textbf{FAPE-IR}) consistently achieves state-of-the-art or comparable performance across diverse restoration benchmarks.}
\label{figs:res2}
\begin{adjustbox}{max width=\textwidth}
\begin{tabular}{@{} c *{20}{c} @{}}
\toprule

% ==================== Row-block A (Rain100_L | Rain100_H | OutDoor | RainDrop) ====================
\FourDatasetHeaderIQA{OutDoor}{RainDrop}{Rain100\_L}{Rain100\_H}

% —— Section: Classical all-in-one IR (orange band)
\SectionBand{BandOrange}{State-of-the-art AIO-IR methods}
PromptIR~\cite{potlapalli2306promptir} & 4.59 & 69.40 & 0.68 & 0.69 & 0.50 & 6.68 & 61.27 & 0.58 & 0.45 & 0.39 & 7.18 & 59.42 & 0.58 & 0.31 & 0.46 & 3.16 & 65.43 & 0.66 & 0.46 & 0.51 \\
FoundIR~\cite{li2024foundir} & 4.76 & 68.41 & 0.67 & 0.63 & 0.48 & 8.48 & 60.45 & 0.57 & 0.38 & 0.41 & 6.32 & 61.27 & 0.57 & 0.32 & 0.40 & 3.65 & 67.02 & 0.66 & 0.42 & 0.50 \\
DFPIR~\cite{tian2025degradation} & 4.60 & 69.64 & 0.68 & 0.70 & 0.50 & 6.38 & 61.81 & 0.59 & 0.46 & 0.40 & 6.87 & 57.64 & 0.52 & 0.28 & 0.40 & 3.17 & 64.49 & 0.65 & 0.42 & 0.50 \\
MoCE-IR~\cite{zamfir2025complexity}  & 4.59 & 69.74 & 0.68 & 0.70 & 0.50 & 7.20 & 62.47 & 0.59 & 0.46 & 0.41 & 7.11 & 58.90 & 0.58 & 0.29 & 0.47 & 3.21 & 65.56 & 0.66 & 0.45 & 0.52 \\
AdaIR~\cite{cui2024adair} & 4.60 & 69.67 & 0.68 & 0.70 & 0.50 & 6.89 & 61.96 & 0.59 & 0.45 & 0.40 & 7.19 & 58.95 & 0.58 & 0.29 & 0.48 & 3.15 & 65.35 & 0.66 & 0.49 & 0.51 \\
% \SectionBand{BandPurple}{Plan then Restore (Ours)}
\rowcolor{blue!8}
Ours & 5.04 & 69.34 & 0.68 & 0.67 & 0.48 & 5.04 & 66.57 & 0.64 & 0.61 & 0.45 & 3.42 & 68.26 & 0.68 & 0.39 & 0.58 & 2.99 & 69.12 & 0.70 & 0.46 & 0.55 \\
\midrule

% ==================== Row-block B (BSD68-15 | BSD68-25 | BSD68-50 | Urban100-15) ====================
\FourDatasetHeaderIQA{BSD68-15}{BSD68-25}{BSD68-50}{Urban100-15}
\SectionBand{BandOrange}{State-of-the-art AIO-IR methods}
PromptIR~\cite{potlapalli2306promptir} & 5.25 & 45.68 & 0.57 & 0.55 & 0.33 & 5.71 & 39.01 & 0.54 & 0.47 & 0.30 & 5.33 & 72.86 & 0.52 & 0.38 & 0.28 & 5.33 & 72.86 & 0.75 & 0.62 & 0.73 \\
FoundIR~\cite{li2024foundir}  & 5.34 & 46.89 & 0.57 & 0.58 & 0.34 & 5.80 & 40.08 & 0.55 & 0.50 & 0.31 & 7.32 & 37.11 & 0.52 & 0.43 & 0.30 & 4.25 & 71.22 & 0.72 & 0.61 & 0.67 \\
DFPIR~\cite{tian2025degradation}    & 5.27 & 48.63 & 0.57 & 0.60 & 0.35 & 5.66 & 42.49 & 0.55 & 0.54 & 0.33 & 7.05 & 39.04 & 0.52 & 0.45 & 0.32 & 5.57 & 73.52 & 0.76 & 0.67 & 0.73 \\
MoCE-IR~\cite{zamfir2025complexity}  & 5.38 & 46.03 & 0.57 & 0.55 & 0.33 & 5.80 & 39.04 & 0.55 & 0.47 & 0.30 & 7.12 & 35.44 & 0.52 & 0.39 & 0.28 & 5.46 & 73.08 & 0.75 & 0.63 & 0.74 \\
AdaIR~\cite{cui2024adair}    & 5.30 & 45.57 & 0.57 & 0.56 & 0.33 & 5.76 & 38.84 & 0.55 & 0.46 & 0.30 & 7.11 & 35.16 & 0.52 & 0.39 & 0.28 & 5.20 & 72.78 & 0.75 & 0.63 & 0.73 \\
% \SectionBand{BandPurple}{Plan then Restore (Ours)}
\rowcolor{blue!8}
Ours     & 5.13 & 51.60 & 0.58 & 0.64 & 0.36 & 4.87 & 49.24 & 0.56 & 0.50 & 0.34 & 4.64 & 44.90 & 0.74 & 0.64 & 0.70 & 4.90 & 72.17 & 0.74 & 0.64 & 0.70 \\
\midrule

% ==================== Row-block C (Urban100-25 | Urban100-50 | ITS-val | URHI) ====================
\FourDatasetHeaderIQA{Urban100-25}{Urban100-50}{ITS-val}{URHI}
\SectionBand{BandOrange}{State-of-the-art AIO-IR methods}
PromptIR~\cite{potlapalli2306promptir} & 4.79 & 72.39 & 0.74 & 0.61 & 0.72 & 4.11 & 70.23 & 0.68 & 0.56 & 0.68 & 4.53 & 51.32 & 0.50 & 0.10 & 0.35 & 4.21 & 54.28 & 0.64 & 0.27 & 0.35 \\
FoundIR~\cite{li2024foundir}  & 3.86 & 66.26 & 0.69 & 0.58 & 0.61 & 4.34 & 55.79 & 0.63 & 0.54 & 0.50 & 5.65 & 49.82 & 0.50 & 0.14 & 0.36 & 5.16 & 54.44 & 0.63 & 0.30 & 0.35 \\
DFPIR~\cite{tian2025degradation}    & 5.86 & 73.15 & 0.76 & 0.64 & 0.69 & 7.13 & 69.97 & 0.71 & 0.58 & 0.56 & 5.10 & 46.81 & 0.48 & 0.11 & 0.34 & 4.16 & 54.11 & 0.63 & 0.24 & 0.35 \\
MoCE-IR~\cite{zamfir2025complexity}  & 5.54 & 73.36 & 0.75 & 0.60 & 0.73 & 4.80 & 72.46 & 0.73 & 0.56 & 0.69 & 5.01 & 49.54 & 0.49 & 0.11 & 0.36 & 4.24 & 54.34 & 0.63 & 0.27 & 0.36 \\
AdaIR~\cite{cui2024adair}    & 5.01 & 72.75 & 0.75 & 0.62 & 0.72 & 4.13 & 69.45 & 0.68 & 0.58 & 0.67 & 4.65 & 50.44 & 0.49 & 0.10 & 0.34 & 4.22 & 54.20 & 0.63 & 0.28 & 0.35 \\
% \SectionBand{BandPurple}{Plan then Restore (Ours)}
\rowcolor{blue!8}
Ours     & 4.92 & 71.90 & 0.73 & 0.62 & 0.69 & 4.96 & 70.73 & 0.70 & 0.57 & 0.66 & 5.66 & 47.88 & 0.55 & 0.13 & 0.30 & 4.47 & 53.76 & 0.64 & 0.23 & 0.35 \\
\midrule

% ==================== Row-block D (Snow100K-L | Snow100K-S | GoPro | GoPro-gamma) ====================
\FourDatasetHeaderIQA{Snow100K-L}{Snow100K-S}{GoPro}{GoPro-gamma}
\SectionBand{BandGreen}{State-of-the-art AIO-IR methods}
AdaIR~\cite{cui2024adair}   & 3.61 & 55.64 & 0.63 & 0.42 & 0.42 & 3.56 & 60.12 & 0.66 & 0.47 & 0.46 & 4.86 & 46.63 & 0.57 & 0.22 & 0.31 & 4.90 & 45.67 & 0.56 & 0.21 & 0.30 \\
FoundIR~\cite{li2024foundir} & 4.53 & 57.99 & 0.64 & 0.39 & 0.42 & 4.56 & 62.57 & 0.67 & 0.43 & 0.46 & 5.24 & 40.86 & 0.49 & 0.22 & 0.26 & 5.25 & 41.06 & 0.50 & 0.22 & 0.26 \\
DFPIR~\cite{tian2025degradation}   & 4.09 & 54.86 & 0.62 & 0.39 & 0.41 & 3.99 & 59.47 & 0.65 & 0.43 & 0.44 & 4.73 & 50.62 & 0.60 & 0.23 & 0.34 & 4.79 & 49.54 & 0.59 & 0.22 & 0.33 \\
MoCE-IR~\cite{zamfir2025complexity} & 3.85 & 55.91 & 0.63 & 0.38 & 0.42 & 3.76 & 60.54 & 0.66 & 0.43 & 0.46 & 4.61 & 51.14 & 0.60 & 0.25 & 0.35 & 4.70 & 48.42 & 0.58 & 0.23 & 0.33 \\
% \SectionBand{BandPurple}{Plan then Restore (Ours)}
\rowcolor{blue!8}
Ours    & 3.66 & 62.51 & 0.66 & 0.41 & 0.45 & 3.55 & 63.57 & 0.67 & 0.43 & 0.47 & 4.33 & 53.83 & 0.61 & 0.24 & 0.39 & 4.32 & 53.91 & 0.60 & 0.24 & 0.40 \\
\midrule

% ==================== Row-block E (RealBlur-J | RealBlur-R | LOL-v2 | LOL-v1) ====================
\FourDatasetHeaderIQA{RealBlur-J}{RealBlur-R}{LOL-v2}{LOL-v1}
\SectionBand{BandGreen}{State-of-the-art AIO-IR methods}
AdaIR~\cite{cui2024adair}   & 5.17 & 42.85 & 0.53 & 0.22 & 0.29 & 5.72 & 41.24 & 0.50 & 0.20 & 0.27 & 4.27 & 63.12 & 0.64 & 0.41 & 0.51 & 4.35 & 70.06 & 0.63 & 0.38 & 0.58 \\
FoundIR~\cite{li2024foundir} & 5.56 & 41.41 & 0.53 & 0.21 & 0.28 & 8.28 & 29.52 & 0.50 & 0.24 & 0.22 & 5.15 & 56.67 & 0.65 & 0.44 & 0.44 & 5.72 & 65.41 & 0.66 & 0.38 & 0.51 \\
DFPIR~\cite{tian2025degradation}   & 5.45 & 48.50 & 0.58 & 0.24 & 0.33 & 8.33 & 28.57 & 0.52 & 0.19 & 0.22 & 4.33 & 64.22 & 0.64 & 0.39 & 0.52 & 4.55 & 69.37 & 0.62 & 0.36 & 0.57 \\
MoCE-IR~\cite{zamfir2025complexity} & 5.23 & 44.07 & 0.53 & 0.23 & 0.30 & 5.99 & 43.54 & 0.51 & 0.23 & 0.29 & 4.27 & 64.95 & 0.65 & 0.43 & 0.52 & 4.47 & 71.48 & 0.64 & 0.42 & 0.60 \\
% \SectionBand{BandPurple}{Plan then Restore (Ours)}
\rowcolor{blue!8}
Ours    & 4.92 & 52.27 & 0.61 & 0.24 & 0.39 & 6.71 & 32.52 & 0.57 & 0.22 & 0.26 & 4.74 & 65.31 & 0.66 & 0.40 & 0.50 & 4.92 & 69.84 & 0.64 & 0.39 & 0.59 \\
\midrule

% ==================== Row-block F (RealSR-2x | DrealSR-2x | RealSR-4x | DrealSR-4x) ====================
\FourDatasetHeaderIQA{RealSR 2$\times$}{DrealSR 2$\times$}{RealSR 4$\times$}{DrealSR 4$\times$}
\SectionBand{BandPink}{State-of-the-art SR methods}
StableSR~\cite{wang2024exploiting} & 6.62 & 63.20 & 0.63 & 0.63 & 0.51 & 6.58 & 60.31 & 0.60 & 0.63 & 0.51 & 5.86 & 58.56 & 0.57 & 0.58 & 0.47 & 6.91 & 51.74 & 0.52 & 0.58 & 0.45 \\
DiffBIR~\cite{lin2024diffbir}  & 5.97 & 69.44 & 0.66 & 0.70 & 0.68 & 5.81 & 66.79 & 0.63 & 0.70 & 0.67 & 5.60 & 69.55 & 0.65 & 0.71 & 0.68 & 7.00 & 65.88 & 0.61 & 0.71 & 0.66 \\
SeeSR~\cite{wu2024seesr}    & 5.71 & 71.46 & 0.67 & 0.72 & 0.72 & 6.05 & 67.93 & 0.65 & 0.70 & 0.70 & 5.47 & 70.43 & 0.65 & 0.70 & 0.71 & 6.36 & 64.96 & 0.60 & 0.69 & 0.67 \\
PASD~\cite{yang2024pixel}     & 5.12 & 67.67 & 0.63 & 0.62 & 0.61 & 5.66 & 65.53 & 0.61 & 0.65 & 0.63 & 5.21 & 64.63 & 0.58 & 0.58 & 0.58 & 6.94 & 57.85 & 0.52 & 0.57 & 0.55 \\
OSEDiff~\cite{wu2024one}  & 5.81 & 70.55 & 0.66 & 0.69 & 0.64 & 6.02 & 67.61 & 0.63 & 0.69 & 0.63 & 5.74 & 69.14 & 0.63 & 0.67 & 0.63 & 6.78 & 62.68 & 0.58 & 0.69 & 0.59 \\
PURE~\cite{wei2025perceive}     & 5.34 & 69.47 & 0.66 & 0.72 & 0.65 & 6.18 & 65.95 & 0.62 & 0.70 & 0.63 & 5.77 & 66.84 & 0.62 & 0.69 & 0.61 & 7.11 & 60.14 & 0.57 & 0.67 & 0.58 \\
% \SectionBand{BandPurple}{Plan then Restore (Ours)}
\rowcolor{blue!8}
Ours     & 7.05 & 51.94 & 0.53 & 0.35 & 0.37 & 7.67 & 47.74 & 0.50 & 0.34 & 0.37 & 7.63 & 52.65 & 0.48 & 0.39 & 0.40 & 8.79 & 46.53 & 0.44 & 0.45 & 0.41 \\
\bottomrule
\end{tabular}
\end{adjustbox}
\vspace{-2ex}
\end{table*}

\section{Detailed Metrics}
To complement the main quantitative comparison in Table~\ref{tab:IR_series} and Table~\ref{tab:SR_series}, we provide full per-benchmark results in Table~\ref{figs:res} and Table~\ref{figs:res2} of the appendix. These tables report all five distortion–perception metrics (PSNR, SSIM, LPIPS, FID, DISTS) for each individual benchmark across seven restoration tasks (deraining, desnowing, dehazing, deblurring, denoising, low-light enhancement, and super-resolution), together with a unified comparison against recent AIO-IR methods~\cite{potlapalli2306promptir,li2024foundir,tian2025degradation,zamfir2025complexity,cui2024adair} and SR baselines~\cite{wang2024exploiting,lin2024diffbir,wu2024seesr,yang2024pixel,wu2024one,wei2025perceive}. Overall, FAPE-IR attains state-of-the-art or comparable performance on the majority of benchmarks, especially on weather-related degradations and challenging real-world benchmarks, where it substantially improves both PSNR/SSIM and perceptual metrics.

In the high-frequency–dominant restoration regimes (e.g., OutDoor, RainDrop, Snow100K-L/S, Rain100\_H, GoPro-gamma, RealBlur-J/R, Urban100-15/25/50), our method consistently attains markedly lower LPIPS and DISTS together with competitive or clearly higher PSNR/FID than prior AIO-IR approaches, indicating that the frequency-aware planner and band-specialized LoRA-MoE executor effectively suppress artifacts while preserving fine structures. In contrast, for low-frequency degradations such as URHI, ITS-val, and LOL-v1/v2, FAPE-IR achieves strong gains in FID and DISTS while maintaining high SSIM, reflecting better global contrast and color consistency under severe haze and illumination changes. A remaining challenge lies in Rain100\_L and the BSD68-15/25/50 denoising benchmarks, where FAPE-IR is less competitive in PSNR. We attribute this mainly to a mismatch between our training distribution and these relatively simple degradations (light rain and synthetic Gaussian noise), together with the limited amount of pure denoising data: the planner tends to favor stronger high-frequency suppression, which can lead to slight over-smoothing and thus lower PSNR on texture-rich scenes, while still preserving favorable perceptual metrics.

Table~\ref{figs:res2} further reports no-reference image quality metrics on the same set of benchmarks. Across most restoration benchmarks (excluding SR), FAPE-IR improves or matches the best scores on at least three of the five IQA indicators, confirming that our adversarial training and frequency regularization translate into reconstructions that are also preferred by hand-crafted NR-IQA measures. For real-world super-resolution, however, FAPE-IR exhibits a different pattern: while Table~\ref{figs:res} shows clear advantages in full-reference distortion and perceptual metrics (PSNR/SSIM, LPIPS, FID, DISTS), the no-reference scores in Table~\ref{figs:res2} are sometimes worse than those of competing SR methods. We speculate that this discrepancy stems from the SR ground-truth images themselves, whose statistics yield relatively poor scores under standard NR-IQA metrics. By recovering textures and statistics closer to these ground truths, FAPE-IR improves full-reference and learned perceptual metrics but can be penalized by hand-crafted no-reference indicators. Overall, the tables demonstrate that FAPE-IR maintains robust distortion–perception trade-offs across diverse degradations and benchmarks, while also highlighting the importance of higher-quality SR training data and more reliable NR-IQA models for future work, especially on challenging benchmarks such as Urban100 and SR benchmarks.

% \section{Other Visualization Results}
\section{Other Visualizations and FM Analysis}
\label{sec:more_vis}

In this section, we provide additional qualitative results that are not included in the main paper due to space limitations. These visualizations cover a broader range of degradations and benchmarks, further illustrating the effectiveness and generalization ability of our FAPE-IR. For clarity, we group the results into two parts: (1) supplementary qualitative comparisons, and (2) visualizations derived from our flow-matching training process.

\begin{figure*}[t]
\centering
\includegraphics[width=\linewidth]{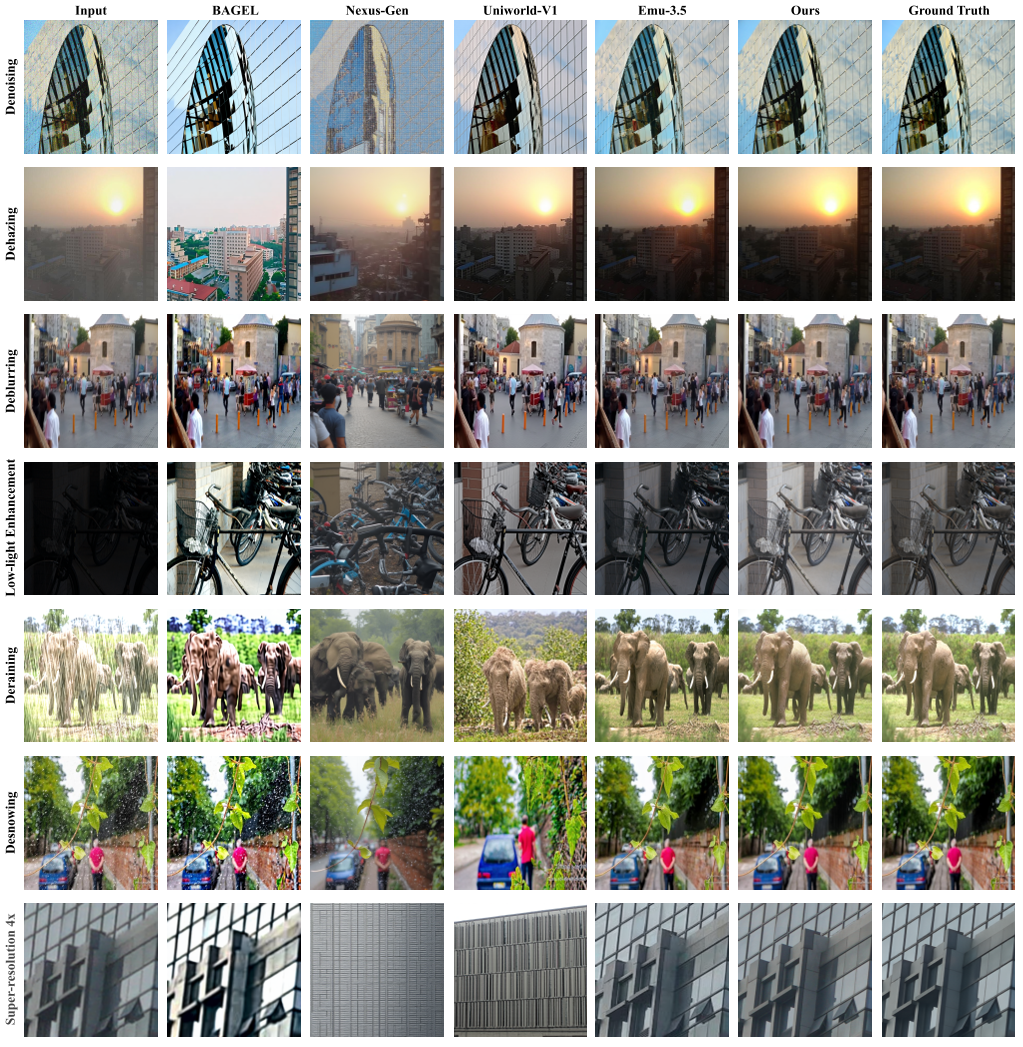}
\caption{Comparison among unified models, including BAGEL~\cite{deng2025emerging}, Nexus-Gen~\cite{zhang2504nexus}, Uniworld-V1~\cite{lin2025uniworld}, and Emu3.5~\cite{cui2025emu3}.}
\label{fig:qual_appendix_um1}
\end{figure*}

\begin{figure*}[t]
\centering
\includegraphics[width=\linewidth]{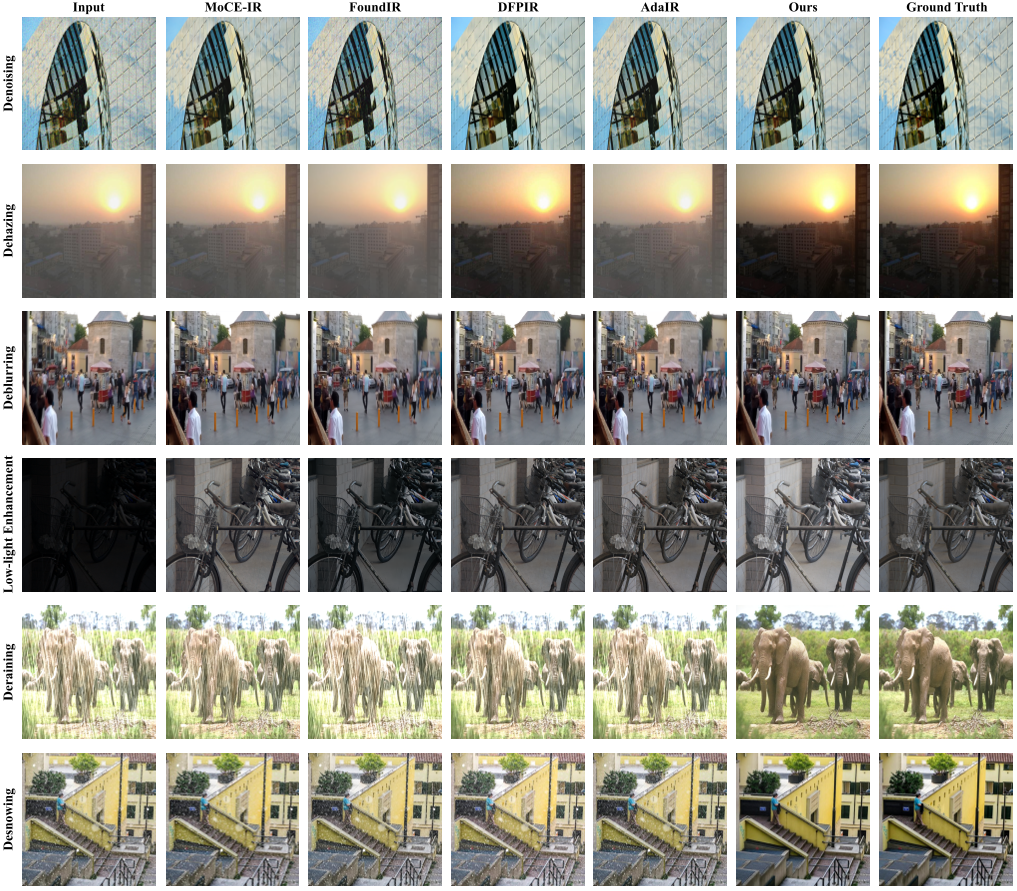}
\caption{Qualitative comparison of restoration results produced by FAPE-IR and state-of-the-art AIO-IR models.}
\label{fig:qual_appendix_um2}
\end{figure*}

\subsection{More Qualitative Comparisons}
Figures~\ref{fig:qual_appendix_um1}–\ref{fig:qual_appendix_um2} present additional qualitative comparisons on deraining, desnowing, deblurring, dehazing, denoising, low-light enhancement, and real-world super-resolution. Figure~\ref{fig:qual_appendix_um1} compares FAPE-IR with recent unified models, while Figure~\ref{fig:qual_appendix_um2} focuses on strong AIO-IR baselines. Across all tasks, FAPE-IR introduces fewer artifacts, better respects the input content, and more faithfully preserves high-frequency details than both categories of baselines.

\begin{itemize}
\item \textbf{Rainy scenes (deraining):} Unified models often hallucinate textures or alter scene layouts, and AIO-IR methods tend to leave residual streaks or over-smooth details. In contrast, FAPE-IR suppresses rain streaks more thoroughly while retaining local texture contrast, consistent with the improvements in LPIPS and DISTS.
\item \textbf{Snow scenes (desnowing):} Competing methods either fail to fully remove veiling snow or over-smooth the background. FAPE-IR removes both small particles and large translucent flakes while maintaining the underlying structures of objects such as buildings and vegetation.
\item \textbf{Denoising:} Under heavy Gaussian noise, unified models may introduce unnatural textures, whereas AIO-IR baselines can blur fine patterns. FAPE-IR better preserves sharp edges and regular patterns (e.g., window grids) with fewer color shifts and blotchy artifacts.
\item \textbf{Real-world blur (deblurring):} The proposed planner more effectively separates structural edges from noise-like blur, leading to sharper reconstructions with substantially fewer ringing and overshoot artifacts than both unified models and AIO-IR methods.
\item \textbf{Dehazing and low-light images:} For hazy and under-exposed scenes, baseline methods often exhibit residual haze, elevated noise, or strong color bias. FAPE-IR produces smoother illumination transitions, more balanced global contrast, and more natural color tones.
\item \textbf{Super-resolution:} On real-world $\times 4$ super-resolution, unified models sometimes hallucinate high-frequency patterns that deviate from the ground truth, while AIO-IR methods tend to over-smooth repetitive structures. FAPE-IR reconstructs sharper, more regular textures (e.g., building facades) that better match the ground-truth distribution, despite the NR-IQA metrics being challenging and sometimes misaligned with human perception.
\end{itemize}

These additional qualitative comparisons further corroborate the distortion–perception trade-offs discussed in the main paper and illustrate that FAPE-IR scales more reliably than both unified models and existing AIO-IR approaches across diverse degradation types.

\begin{figure*}[t]
\centering
\includegraphics[width=\linewidth]{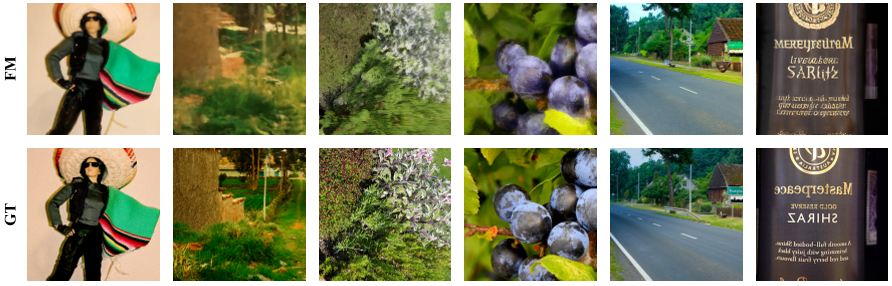}
\caption{Qualitative results of training our framework with a standard flow-matching (FM) objective on real-world super-resolution. Although the FM-trained variant can sharpen some structures, it also introduces severe artifacts and unrealistic high-frequency details (e.g., distorted edges and hallucinated textures), which motivates our final design choices for FAPE-IR.}
\label{fig:flow_vis_fm_ll}
\vspace{-2ex}
\end{figure*}

\subsection{Flow-Matching Visualizations}
\label{sec:fm_vis}
In our early exploratory stage, we attempted to train the entire FAPE-IR framework using a standard flow-matching (FM) objective~\cite{lipman2022flow}. Figure~\ref{fig:flow_vis_fm_ll} shows representative qualitative results on real-world super-resolution benchmarks. While the FM-trained variant is able to remove part of the degradation and produce sharper outputs than the input LR images, it also generates a large number of artifacts and unrealistic details: building facades exhibit irregular, “painted” textures, edges become locally distorted, and fine patterns are often hallucinated rather than faithfully reconstructed from the input.
These failure cases provide an empirical counterexample to the naive expectation that FM alone is sufficient for all-in-one restoration in pixel space. In ill-posed tasks such as real-world SR, the learned flow tends to overfit the training distribution and prioritize distribution matching over content preservation, leading to spurious high-frequency components that hurt perceptual realism. This observation is consistent with our theoretical motivation, and it prompted us to explore additional adversarial training and frequency-aware regularization to better constrain the planner–executor pipeline.

Although this FM-based variant is discarded in our final system, we include it here for completeness, given the widespread use of flow-matching in recent unified generative models. In future work, we plan to investigate ways to mitigate these artifacts, e.g., by incorporating stronger structural priors or geometry-aware constraints into the flow field, so that FM can be more safely integrated into general-purpose restoration frameworks like FAPE-IR.

\subsection{Hyperparameter Sensitivity Analysis}
\label{sec:appendix_hparam}

\begin{figure}[t]
\centering
\includegraphics[width=\linewidth]{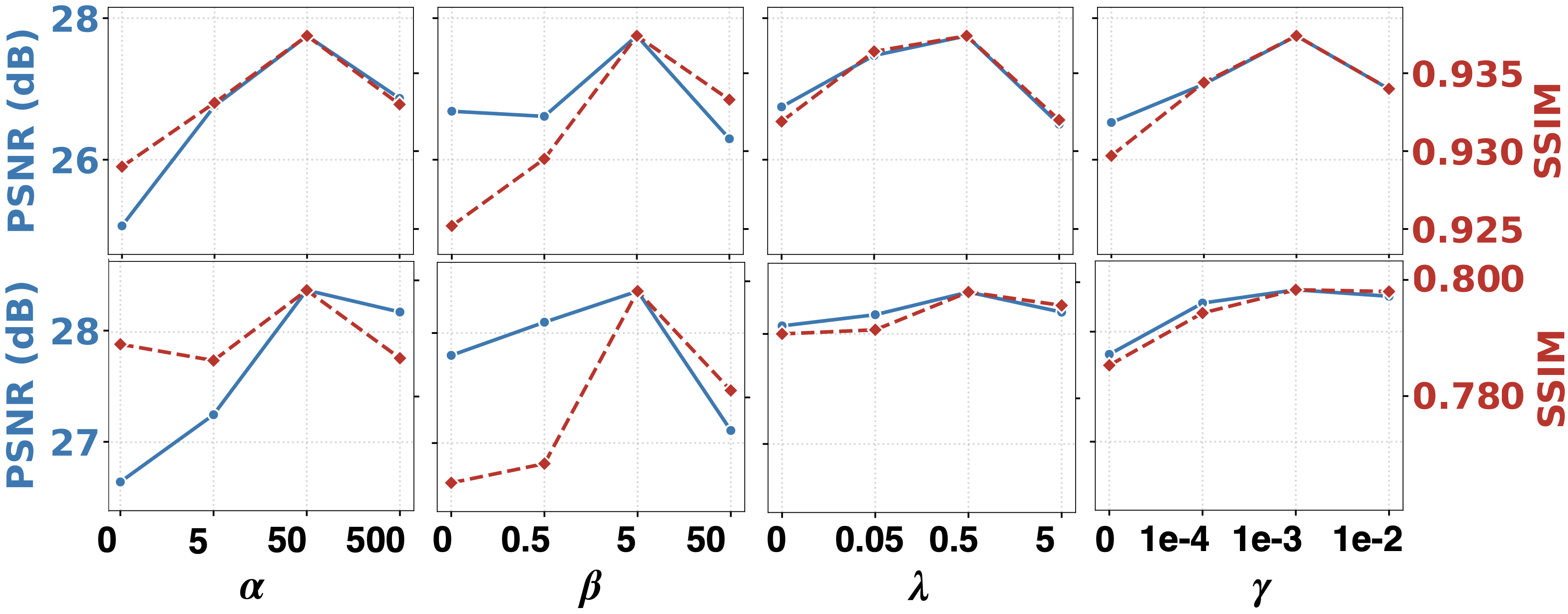}
\caption{Hyperparameter sensitivity analysis of the four loss weights $\alpha$, $\beta$, $\lambda$, and $\gamma$ on URHI and BSD68-15.}
\label{fig:trends}
\vspace{-3ex}
\end{figure}

To further validate the robustness of our objective design, we provide a more detailed hyperparameter sensitivity analysis for the four loss weights, namely $\alpha$, $\beta$, $\lambda$, and $\gamma$. Specifically, for each loss term, we vary its weight to $\{0\times, 0.1\times, 1\times, 10\times\}$ relative to the default setting, while keeping all remaining loss weights fixed. For efficiency, all variants are trained for 5k iterations under the same training protocol.
Figure~\ref{fig:trends} reports the corresponding results on two representative benchmarks, i.e., URHI and BSD68-15, in terms of both PSNR and SSIM. Overall, the results show a clear and consistent trend: each loss term contributes positively to the final performance, and removing any of them (i.e., setting the corresponding weight to $0$) leads to noticeable degradation. This confirms that the four components are complementary rather than redundant.

\begin{table}[t]
  \centering
  \small
  \caption{Ablation on BSD68-15 benchmark. Qwen: remove Qwen2.5-VL; Freq-U: remove frequency-aware text router; Freq-G: remove FIR spectral router; $r$: LoRA rank in expert combos.}
  \label{tab:bsd15_hf_ablation}
  \begin{adjustbox}{scale=0.75}
  \begin{tabular}{cccccccc}
    \toprule
    \textbf{Qwen} & \textbf{Freq-U} & \textbf{Freq-G} & \textbf{r=4} & \textbf{r=8} & \textbf{r=16} &
    \textbf{PSNR$\uparrow$} & \textbf{SSIM$\uparrow$} \\
    \midrule
    \xmark & \xmark & \xmark &  & \cmark &  & 30.81 & 0.90 \\
    \cmark & \xmark & \xmark &  & \cmark &  & 31.19 & 0.90 \\
    \addlinespace[2pt]

    \multicolumn{8}{l}{\textit{+ Freq-U enabled}} \\
    \cmark & \cmark & \xmark &  & \cmark &  & 31.27 & 0.90 \\
    \cmark & \cmark & \xmark &  & \cmark \cmark &  & 31.69 & 0.90 \\
    \cmark & \cmark & \xmark & \cmark &  & \cmark & 31.55 & 0.91 \\
    \cmark & \cmark & \xmark &  & \cmark & \cmark & 32.41 & 0.91 \\
    \addlinespace[2pt]

    \multicolumn{8}{l}{\textit{+ Freq-G enabled}} \\
    \cmark & \cmark & \cmark &  & \cmark \cmark &  & 33.57 & 0.91 \\
    \cmark & \cmark & \cmark & \cmark &  & \cmark & 33.12 & 0.91 \\
    \cmark & \cmark & \cmark &  & \cmark & \cmark & 33.20 & 0.91 \\
    \bottomrule
  \end{tabular}
  \end{adjustbox}
  \vspace{-3ex}
\end{table}

\subsection{Ablation on High-Frequency Restoration}
\label{sec:appendix_hf_ablation}
In the main text, we present the structure ablation results on the low-frequency restoration setting. To provide a more complete picture, we further report the corresponding ablation results on a high-frequency benchmark, i.e., BSD68-15, in Table~\ref{tab:bsd15_hf_ablation}.
The overall trend is consistent with the observations in the main text. First, introducing Qwen2.5-VL already brings a clear improvement over the base model, indicating that semantic guidance from the vision-language planner is beneficial not only for low-frequency recovery but also for more challenging high-frequency restoration. Second, enabling the frequency-aware text router (Freq-U) further improves performance, showing that frequency-conditioned text routing helps better align semantic priors with the restoration process. Third, adding the FIR spectral router (Freq-G) leads to the largest gain, which confirms the importance of explicit frequency-aware expert selection in handling complex high-frequency degradations.

We also compare different LoRA rank combinations. Among all tested settings, the combination centered on $r=8$ achieves the best overall performance, while configurations involving $r=4$ or $r=16$ are generally less effective. This suggests that a moderate LoRA rank provides a better trade-off between capacity and specialization for expert routing.
Overall, these supplementary results on BSD68-15 further validate that each proposed component contributes positively and that the conclusions drawn from the low-frequency setting generalize well to the high-frequency restoration scenario.